\documentclass{article} 
\usepackage{iclr2022_conference_arxiv,times}

\usepackage{amsmath,amsfonts,bm}

\def\eqref#1{equation~\ref{#1}}

\def\1{\bm{1}}

\DeclareMathAlphabet{\mathsfit}{\encodingdefault}{\sfdefault}{m}{sl}
\SetMathAlphabet{\mathsfit}{bold}{\encodingdefault}{\sfdefault}{bx}{n}

\usepackage[pdftex]{graphicx}

\iclrfinalcopy 

\usepackage[utf8]{inputenc} 
\usepackage[T1]{fontenc}    
\usepackage[hidelinks]{hyperref}       
\usepackage{algorithm}
\usepackage{algorithmic}
\usepackage{url}            
\usepackage{booktabs}       
\usepackage{amsfonts}       
\usepackage{nicefrac}       
\usepackage{microtype}      
\usepackage{mathtools}
\usepackage{titlesec}
\usepackage[capitalize]{cleveref}
\usepackage{tikz}
\usepackage{subcaption}
\usepackage[symbol]{footmisc}

\usetikzlibrary{shapes,decorations,arrows,calc,arrows.meta,fit,positioning}
\tikzset{
	-Latex,auto,node distance =1 cm and 1 cm,semithick,
	state/.style ={ellipse, draw, minimum width = 0.7 cm},
	point/.style = {circle, draw, inner sep=0.04cm,fill,node contents={}},
	bidirected/.style={Latex-Latex,dashed},
	el/.style = {inner sep=2pt, align=left, sloped}
}

\titlespacing{\subsection}{0pt}{\parskip}{-\parskip}
\usepackage{amsthm}
\newtheorem{thm}{Theorem}[section]

\newtheorem{defn}[thm]{Definition}

\usepackage{xpatch}
\makeatletter
\newenvironment{proof*}[1][\proofname]{\par
	\pushQED{\qed}%
	\normalfont \partopsep=\z@skip \topsep=\z@skip
	\trivlist
	\item[\hskip\labelsep
	\itshape
	#1\@addpunct{.}]\ignorespaces
}{%
	\popQED\endtrivlist\@endpefalse
}

\usepackage[title]{appendix}

\title{SoftHebb: Bayesian Inference in Unsupervised Hebbian Soft Winner-Take-All Networks}

\begin{document}
	\maketitle
\begin{table}[ht]
	\centering
	\begin{tabular}{c}
		\textbf{Timoleon Moraitis}$\boldsymbol{^*}$\\
		Huawei Technologies \\
		Zurich Research Center, Switzerland\\
		\rule{\z@}{24pt}\\
	\end{tabular}
	\begin{tabular}{cc}
	\textbf{Dmitry Toichkin} & \textbf{Adrien Journé}\\
	Moscow & Huawei Technologies \\
	Russia & Zurich Research Center, Switzerland \\
	\rule{\z@}{24pt}\\
	\textbf{Yansong Chua} & \textbf{Qinghai Guo}$\boldsymbol{^*}$\\
	Huawei Technologies & Huawei Technologies \\
	ACS Lab, Shenzhen, China & ACS Lab, Shenzhen, China\\
	\end{tabular}
\end{table}
\vskip 0.2in minus 0.1in
\footnotetext{$^*$Corresponding authors. \{\texttt{timoleon.moraitis}, \texttt{guoqinghai}\} \texttt{@huawei.com}}

	\begin{abstract}
		Hebbian plasticity in winner-take-all (WTA) networks is highly attractive for neuromorphic on-chip learning, owing to its efficient, local, unsupervised, and on-line nature. Moreover, its biological plausibility may help overcome important limitations of artificial algorithms, such as their susceptibility to adversarial attacks, and their high demands for training-example quantity and repetition. However, Hebbian WTA learning has found little use in machine learning (ML), likely because it has been missing an optimization theory compatible with deep learning (DL). Here we show rigorously that WTA networks constructed by standard DL elements, combined with a Hebbian-like plasticity that we derive, maintain a Bayesian generative model of the data. Importantly, without any supervision, our algorithm, SoftHebb, minimizes cross-entropy, i.e.\ a common loss function in supervised DL. We show this theoretically and in practice. The key is a ``soft'' WTA where there is no absolute ``hard'' winner neuron.
		Strikingly, in shallow-network comparisons with backpropagation (BP), SoftHebb shows advantages beyond its Hebbian efficiency. Namely, it converges in fewer iterations, and is significantly more robust to noise and adversarial attacks. Notably, attacks that maximally confuse SoftHebb are also confusing to the human eye, potentially linking human perceptual robustness, with Hebbian WTA circuits of cortex. Finally, SoftHebb can generate synthetic objects as interpolations of real object classes. All in all, Hebbian efficiency, theoretical underpinning, cross-entropy-minimization, and surprising empirical advantages, suggest that SoftHebb may inspire highly neuromorphic and radically different, but practical and advantageous learning algorithms and hardware accelerators.
\end{abstract}

\section{Introduction}
State-of-the-art (SOTA) artificial neural networks (ANNs) achieve impressive results in a variety of machine intelligence tasks \citep{sejnowski2020unreasonable}. However, they largely rely on mechanisms that diverge from the original inspiration from biological neural networks \citep{bengio2015towards, illing2019biologically}. As a result, only a small part of this prolific field also contributes to computational neuroscience. In fact, biological implausibility is also an important issue for machine intelligence. Despite their impressive performance, ANNs often neglect properties that are present in biological systems, and these properties could offer a path to the next generation of artificial intelligent systems \citep{zador2022toward}.
Namely, neuromorphic computing has been advancing machine intelligence in energy efficiency, and recent evidence shows that it improves also conventional metrics of state-of-the-art performance such as accuracy, reward, or speed. For example, spike-based models achieve processing speed and energy efficiency through the imitation of biological neuronal activations, without trading off performance \citep{jeffares2022spikeinspired}, or with minimal trade-offs \citep{bittar2022surrogate}; short-term plasticity improves the performance of neural networks in dynamic tasks such as video processing, navigation, robotics, and video games \citep{moraitis2020shortterm, rodriguez2022short}; efference copies advance self-supervised learning \citep{scherr2022self}; and dendritic computations increase the computational power of individual neurons \citep{poirazi2020illuminating, sarwat2022chalcogenide}.
In this work instead we focus on a different neuromorphic aspect, namely a synaptic plasticity mechanism for learning that is strongly supported by biological evidence, aiming here as well not only for efficiency but also for other advantages over conventional DL performance.
\begin{figure*}[h]
			\centering
	\includegraphics[width = 1.0\linewidth]{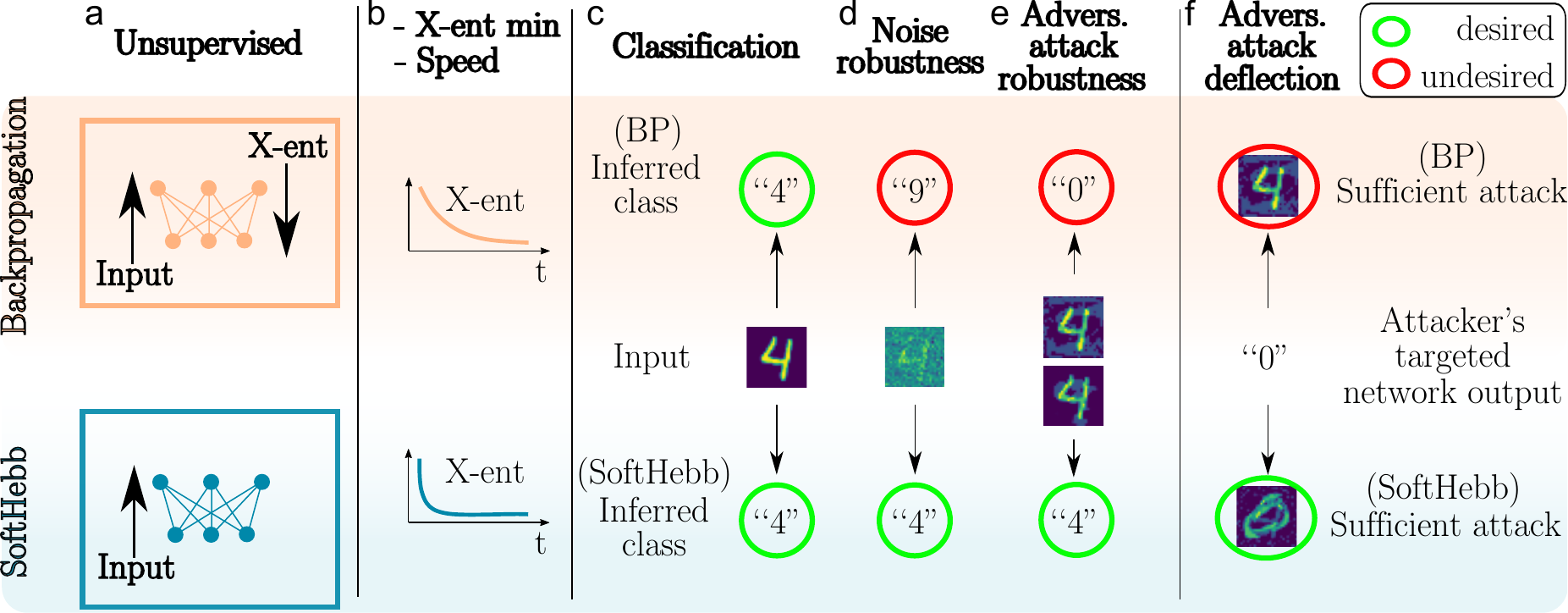}
	\caption{Schematic: summary of the key properties of SoftHebb, contrasted with backpropagation. (a) Unsupervised: SoftHebb uses no supervision by top-down signals, such as cross-entropy (``X-ent''). (b) X-ent minimization \& speed: Nevertheless SoftHebb minimizes the cross-entropy loss under certain assumptions. Moreover, it converges faster than backpropagation in number of learning iterations. (c) Classification: SoftHebb's unsupervised algorithm can be used to cluster an input dataset into classes, either on its own through its Bayesian inference of the classes as hidden causes of the input, or with an added supervised linear classifier. E.g.\ with an unperturbed image of a handwritten digit ``4'' (middle row), networks trained with both backprop (top) and SoftHebb (bottom) perform well and recognize the digit correctly (green circles). (d) Noise robustness: When Gaussian noise is added to the input at inference, the backprop-trained network misclassifies the digit (red circle), whereas SoftHebb is robust. (e) Adversarial attack robustness: With a white-box adversarial attack, the digit's pixels are perturbed to maximize the loss of each specific network, with its parameters known to the attacker. Note that the images (middle row) have subtle changes compared to the original input. The attack results in a different image targetting each network. The attack is successful for the backprop-trained network, which misclassifies the digit (red circle) as a ``zero''. SoftHebb on the other hand remains robust. (f) Adversarial attack deflection: An attacker perturbs the input of class ``four'' targetting a network output of class ``zero''. The attacker chooses an intensity that suffices to ascertain the attack's success with a high probability (see Fig. \ref{fig:robustness}). The attack of the backprop-trained network succeeds while the image still appears as a digit ``four''. In the case of SoftHebb, the attacker must truly generate an image of a digit ``zero'' (green circle) to succeed, i.e.\ SoftHebb deflects this adversarial attack attempt.}
	\label{fig:schematic}
\end{figure*}
\paragraph*{Inefficiencies of conventional Deep Learning.} Several limitations of conventional DL appear to be in contrast with some biological learning processes, and could therefore potentially be addressed by neuromorphic learning algorithms. For instance, ANN training often demands very large and labelled datasets, which are costly to generate. When labels are unavailable, self-supervised learning schemes exist, where supervisory error signals generated by the network itself are exploited and backpropagated from the output towards the input to update the network's parameters \citep{goodfellow2014generative, devlin2018bert, chen2020simple, bardes2021vicreg, scherr2022self}. However, this global propagation of signals in deep networks introduces another limitation. Namely, it prevents the implementation of efficient distributed computing hardware that would be based on only local signals from neighbouring physical nodes in the network, it requires teaching currents to flow throughout the network, and is in contrast to the local synaptic plasticity rules that partly govern biological learning.
Several pieces of work have been addressing parts of the biological implausibility and hardware-inefficiency of BP in ANNs \citep{crick1989recent, bengio2015towards, lillicrap2016random, nokland2016direct, guerguiev2017towards, pfeiffer2018deep, illing2019biologically, pogodin2020kernelized, millidge2020predictive, pogodin2021towards, payeur2021burst}, such as requirements of exactly symmetric forward and backward weights or the waiting time caused by the network's forward-backward pass between two training updates in a layer. These are known as the weight-transport \citep{grossberg1987competitive, lillicrap2016random} and update-locking \citep{czarnecki2017understanding, frenkel2021learning} problems of BP.
Recently, an approximation to BP that is mostly Hebbian, i.e.\ relies mostly on pre- and post-synaptic activity of each synapse, has been achieved by reducing the global error requirements to 1-bit information \citep{pogodin2020kernelized}. Two schemes that further localize the signal that is required for a weight update are Equilibrium Propagation \citep{scellier2017equilibrium} and Predictive Coding \citep{millidge2020predictive}. Both methods approximate BP through Hebbian-like learning, by delegating the global aspect of the computation, from a global error signal, to a global convergence of the network state to an equilibrium. This equilibrium is reached through several iterative steps of feed-forward and feed-back communication throughout the network, before the ultimate weight update by one training example. The biological plausibility and hardware-efficiency of this added iterative process of signal propagation are open questions that begin to be addressed \citep{ernoult2020equilibrium}. Therefore, even though there has been significant progress in dealing with some of the inefficiencies and biological implausibilities of BP, this has not been entirely possible, because these approaches aim to approximate BP, rather than learn with a radically different mechanism.

\paragraph*{Adversarial attacks of ANNs. Deflection by humans.} Moreover, learning through BP, and presumably also its approximations, has another indication of biological implausibility, which also significantly limits ANN applicability. Namely, it produces networks that are confused by small adversarial perturbations of the input, which are imperceptible by humans.
It has recently been proposed that a defence strategy of ``deflection'' of adversarial attacks may be the ultimate solution to that problem \citep{qin2020deflecting}. Through this strategy, to cause confusion in the network's inferred class, the adversary is forced to generate such a changed input that it really belongs to the distribution mode of a different input class. Intuitively, but also strictly by definition, this deflection is achieved if a human assigns to the perturbed input the same label that the network does.
Deflection of adversarial attacks in ANNs has been demonstrated by an elaborate scheme that is based on \textit{detecting} the attacks \citep{qin2020deflecting}. However, the human ability to deflect adversarial perturbations likely does not rely on detecting them, but rather on effectively ignoring them, making the deflecting type of robustness an emergent property of biological computation rather than a defence mechanism.
The principles that underlie this biological robustness are unclear, but it might emerge from the distinct algorithms that govern learning in the brain.

\paragraph*{Hebbian WTA.}
Therefore, what is missing is a biologically plausible model that can learn from fewer data-points, without labels, through local plasticity, and without feedback from distant layers. This model could then be tested for emergent adversarial robustness and deflection of adversarial attacks.
A good candidate category of biological networks and learning algorithms is that of competitive learning. Neurons that compete for their activation through lateral inhibition are a common connectivity pattern in the superficial layers of the cerebral cortex \citep{douglas2004neuronal, binzegger2004quantitative}. This pattern is described as winner-take-all (WTA), because competition suppresses activity of weakly activated neurons, and emphasizes strong ones. WTA competition is generally categorized into two types, namely hard WTA, where the winning neuron is the only one active, and soft WTA, where the non-winning neurons are not fully suppressed \citep{binas2014learning}. Combined with Hebbian-like plasticity rules, i.e.\ update rules based on correlated pre- and post-synaptic activity, WTA connectivity gives rise to competitive-learning algorithms. These networks and learning schemes have been long studied \citep{von1973self} and a large literature based on simulations and analyses describes their functional properties. A WTA neuronal layer, depending on its specifics, can restore missing input signals \citep{rutishauser2011collective, diehl2016learning}, perform decision making i.e.\ winner selection \citep{hahnloser1999feedback, maass2000computational, rutishauser2011collective}, and generate oscillations such as those that underlie brain rhythms \citep{cannon2014neurosystems}. Perhaps more importantly, its neurons can learn to become selective to different input patterns, such as orientation of visual bars in models of the primary visual cortex \citep{von1973self}, MNIST handwritten digits \citep{nessler2013PLoS, diehl2015unsupervised, krotov2019unsupervised}, CIFAR-10 objects \citep{krotov2019unsupervised}, spatiotemporal spiking patterns \citep{nessler2013PLoS}, and can adapt dynamically to model changing objects \citep{moraitis2020shortterm}.
The WTA model is indeed biologically plausible, Hebbian plasticity is local, and learning is input-driven, relying on only feed-forward communication of neurons -- properties that seem to address several of the limitations of ANNs. However, the model's applicability is limited to simple tasks. That is partly because the related theoretical literature remains surprisingly unsettled, despite its long history, and the strong and productive community interest \citep{sanger1989optimal, foldiak1989adaptive, foldiak1990forming, linsker1992local, olshausen1996emergence, bell1995information, olshausen1997sparse, lee1999independent, nessler2013PLoS, pehlevan2014hebbian, hu2014hebbian, PehlevanNIPS2015, pehlevan2017clustering, isomura2018error}.

\paragraph*{Necessity for new theoretical foundation of Hebbian WTA learning.}
A very relevant theory in this direction was described by \citet{nessler2009stdp, nessler2013PLoS}. That work showed that winner-take-all circuits implement Bayesian computation and, combined with local plasticity, they implement a type of expectation-maximization. The specific plasticity rule for a synapse connecting a presynaptic neuron $i$ with activation $x_i$ to a postsynaptic neuron $k$ with WTA output $y_k$ in that case was
\begin{equation}
   \Delta w_{ik}^{(Nessler)}= 
\begin{cases}
	\eta\left(e^{-w_{ki}}-1\right),& \text{if } x_i=1 \text{ and } y_k=1 \\
	-\eta,& \text{if } x_i=0 \text{ and } y_k=1 \\
	0,              & \text{if } y_k=0.
\end{cases}
\end{equation}

However, that theory concerned WTA models that are largely incompatible with ANNs and thus less practical. Namely, it assumed spiking and stochastic neurons, input values have to be discretized, and each individual input feature to a layer must be encoded through special population coding by multiple binary neurons. Moreover, it was only proven for neurons with an exponential activation function. It remains therefore unclear which specific plasticity rule and structure could optimize an ANN WTA for Bayesian inference. It is also unclear how to minimize a common loss function such as cross-entropy despite unsupervised learning, and how a WTA could represent varying families of probability distributions. In summary, on the theoretical side, an algorithm that is simultaneously normative, based on WTA networks and Hebbian unsupervised plasticity, performs Bayesian inference, and, importantly, is composed of conventional ANN elements, with conventional input encoding, and is rigorously linked to modern ANN tools such as cross-entropy loss, would be an important advance but has been missing. On the practical side, such a theoretically grounded approach could be the key missing piece for bringing the multiple efficiency facets of biological learning to deep learning.
That is, it could provide insights into how a WTA microcircuit could participate in larger-scale computation by deep cortical or artificial networks.
Furthermore, such a theoretical foundation could also reveal unknown advantages of Hebbian plasticity in WTA networks.
Recently, when WTA networks were studied in a theoretical framework compatible with conventional machine learning (ML), but in the context of short-term (STP) as opposed to long-term Hebbian plasticity, it did result in surprising practical advantages over supervised ANNs \citep{moraitis2020shortterm}, and was followed-up by showing significant benefits also in other networks and multiple advanced tasks \citep{rodriguez2022short}. Theoretical grounding may therefore go beyond merely narrowing the accuracy gap from BP in simple benchmarks, and even indicate scenarios where Hebbian plasticity outperforms.

\paragraph*{Our contributions.}
To this end, here we present a mostly theoretical work on Hebbian WTA networks. We construct ``SoftHebb'', a biologically plausible WTA model that is based on standard rate-based neurons as in ANNs, can accommodate various activation functions, and learns without labels, using local plasticity and only feed-forward communication, i.e.\ the properties we seek for efficient learning that is compatible with DL in ANNs.
Importantly, SoftHebb is equipped with a simple normalization of the layer's activations, and an optional temperature-scaling mechanism \citep{hinton2015distilling}, producing a soft WTA instead of selecting a single ``hard'' winner neuron. This allows us to prove formally that a SoftHebb layer is a generative mixture model that objectively minimizes its Kullback-Leibler (KL) divergence from the input distribution through Bayesian inference, thus providing a new formal ML-theoretic perspective of these networks. An important corollary that we derive is that the layer minimizes its cross-entropy with the input distribution.
We complement our main results, which are theoretical, with experiments that are small-scale but produce intriguing results. As a generative model, SoftHebb has a broader scope than classification, but we test it on image classification tasks.
Surprisingly, in addition to overcoming several inefficiencies of BP, the unsupervised WTA model also outperforms a supervised two-layer perceptron in several aspects: learning speed and accuracy in the first presentation of the training dataset, robustness to noisy data and to one of the strongest white-box adversarial attacks, i.e.\ projected gradient descent (PGD) \citep{madry2017towards}, and without any explicit defence. Interestingly, the SoftHebb model also exhibits inherent properties of deflection \citep{qin2020deflecting} of the adversarial attacks, and generates object interpolations.

\section{Theoretical Results}
We will now derive the theory underpinning SoftHebb. The resulting ML-theoretic probabilistic model and the equivalent neural network are summarized in \cref{fig:diagram}, whereas a succinct description of the learning algorithm is provided in \cref{alg:SoftHebb}.

\begin{figure*}[h]
	\includegraphics[width = 1.0\linewidth]{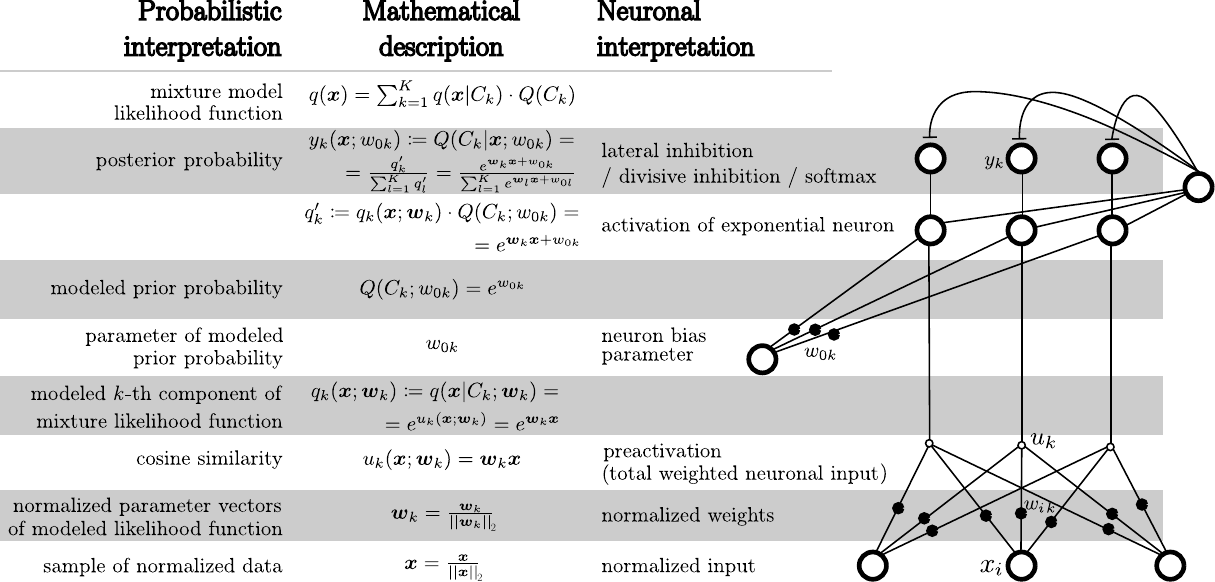}
	\caption{The soft WTA model used in SoftHebb. The network graph is shown on the right. The input to the layer is shown at the bottom and the output is at the top. Each depicted computational element in the diagram is in a white or grey row that also includes the element's description on the left.}
	\label{fig:diagram}
\end{figure*}

\paragraph*{Overview of the derivation.}
The goal of the derivation is to find a WTA neural network and a plasticity rule that optimizes the probabilistic model represented by the network, for a given input distribution. The aimed optimization is specifically the minimization of the KL-divergence between the model distribution and the input distribution, so we aim to derive a plasticity rule that can learn the parameters that achieve this minimum.
The key steps in our derivations are the following. First we define the assumptions about the input and we define a parametrized Bayesian model that will form the backbone of the neural network. Then we determine the optimal parameters of the probabilistic model, i.e.\ those that imply minimum KL divergence from the input. Subsequently, we describe how the model is equivalent to a soft WTA network. Then we describe SoftHebb's plasticity rule, and we show that the parameters that we found as optimal are the plasticity rule's equilibrium, which shows that the plasticity rule updates the network to maintain its Bayesian model optimally. Finally, we describe how given certain assumptions the network minimizes cross-entropy from the input labels, despite the absence of the labels or other supervision. The detailed proofs of the Theorems are provided in the Methods section.

\begin{defn}[\bfseries The input assumptions]
	\label{def:data}
	Each observation $_j\boldsymbol{x} \in \mathbb{R}^n$ is generated by a hidden ``cause'' $_jC$ from a finite set of $K$ possible such causes:
	$_jC \in \{C_k,\, \forall k \leq K\in \mathbb{N}\}.$
	Therefore, the data is generated by a mixture of the probability distributions attributed to each of the $K$ classes $C_k$:
	\begin{equation}
		p(\boldsymbol{x})=\sum_{k=1}^{K}p(\boldsymbol{x}|C_k)P(C_k).
		\label{eq:pstar}
	\end{equation}
	$\boldsymbol{x}$ is a vector quantity, and its dimensions, i.e. components $x_i$, are conditionally independent from each other, i.e.\ $p(\boldsymbol{x})=\prod_{i=1}^{n}p(x_i). \label{eq:independence}$
	The number $K$ of the true causes or classes of the data is assumed to be known.
	\stepcounter{subsection}
\end{defn}
The term ``cause'' is used here in the sense of causal inference. It is important to emphasize that the true cause of each input is hidden, i.e.\ not known. In the case of a labelled dataset, labels commonly correspond to causes, and the labels are deleted before presenting the training data to the model.
We choose a mixture model that corresponds to the data assumptions but is also interpretable in neural terms (Paragraph \ref{sec:neuro_exp}):
\begin{defn}[\bfseries The generative probabilistic mixture model]
	\label{def:model}
	We consider a mixture model distribution $q$:
	$q(\boldsymbol{x})=\sum_{k=1}^{K}q(\boldsymbol{x}|C_k)\,Q(C_k),$
	approximating the data distribution $p$.
	We choose specifically a mixture of exponentials and we parametrize $Q(C_k;w_{0k})$ also as an exponential, specifically:
	\begin{align}
		q(x_i|C_k;w_{ik})&=e^{w_{ik}\cdot \frac{x_i}{||\boldsymbol{x}||_2}},\, \forall k \label{eq:g_param}\\
		Q(C_k;w_{0k})&=e^{w_{0k}},\,\forall k.\label{eq:g0_param}
	\end{align}
	In addition, the parameter vectors are subject to the normalization constraints:
	$||\boldsymbol{w}_k||=1,\, \forall k$,
	and
	$
	\sum_{k=1}^{K}e^{w_{0k}}=1. \label{eq:norm_0}$
	\stepcounter{subsection}
\end{defn}
The model we have chosen is a reasonable choice because it factorizes similarly to the data of Definition \ref{def:data}:
\begin{align}
	\begin{split}
	q_k&\coloneqq q(\boldsymbol{x}|C_k; \boldsymbol{w}_k)=\prod_{i=1}^{n}q(x_i|C_k;w_{ik})\\
	&=e^{\sum_{i=1}^{n}w_{ik}\frac{x_i}{||\boldsymbol{x}||}}=e^{u_k}, \label{eq:multinomial}
	\end{split}
\end{align}
where $u_k=\frac{ \boldsymbol{w}_k\cdot \boldsymbol{x}}{|| \boldsymbol{w}_k||\cdot||\boldsymbol{x}||}$, i.e.\ the cosine similarity of the two vectors.
A similar probabilistic model was used in related previous theoretical work \citep{nessler2009stdp, nessler2013PLoS, moraitis2020shortterm}, but for different data assumptions, and with certain further constraints to the model. Namely, \citet{nessler2009stdp, nessler2013PLoS} considered data that was binary, and created by a population code, while the model was stochastic. These works provide the foundation of our derivation, but here we consider the more generic scenario where data are continuous-valued and input directly into the model, which is deterministic and, as we will show, more compatible with standard ANNs. In \citet{moraitis2020shortterm}, data had particular short-term temporal dependencies, whereas here we consider the distinct case of independent and identically distributed (i.i.d.) input samples.
The Bayes-optimal parameters of a model mixture of exponentials can be found analytically as functions of the input distribution's parameters, and the model is equivalent to a soft winner-take-all neural network \citep{moraitis2020shortterm}. After describing this, we will prove here that Hebbian plasticity of synapses combined with local plasticity of the neuronal biases sets the parameters to their optimal values.

\begin{thm}[\bfseries The optimal parameters of the model] \label{thm:optimal}
	The parameters that minimize the KL divergence of such a mixture model from the data are, for every $k$,
	\begin{gather}
		\prescript{}{opt}{}w_{0k}=\ln P(C_k)
		\label{eq:G0}\\
		\text{and } \prescript{}{opt}{}\boldsymbol{w}^*_k=\frac{ \prescript{}{opt}{}\boldsymbol{w}_k}{|| \prescript{}{opt}{}\boldsymbol{w}_k||}=\frac{\mu_{p_k}\left(\boldsymbol{x}\right)}{||\mu_{p_k}\left(\boldsymbol{x}\right)||},
	\end{gather}
	where $c\in\mathbb{R}^+,\,\prescript{}{opt}{}\boldsymbol{w}_k=c\cdot \mu_{p_k}\left(\boldsymbol{x}\right),\, 
	\mu_{p_k}\left(\boldsymbol{x}\right)$ is the mean of the distribution $p_k$, and $p_k\coloneqq p(\boldsymbol{x}|C_k)$.
	\stepcounter{subsection}
\end{thm}

In other words, the optimal parameter vector of each component $k$ in this mixture is proportional to the mean of the corresponding component of the input distribution, i.e.\ it is a centroid of the component. In addition, the optimal parameter of the model's prior $Q(C_k)$ is the logarithm of the corresponding component's prior probability.
This Theorem's proof was provided in the supplementary material of \citet{moraitis2020shortterm}, but for completeness we also provide it in our Methods section. These centroids and priors of the input's component distributions, as well as the method of their estimation, however, are different for different input assumptions, and we will derive a learning rule that provably sets the parameters to their Maximum Likelihood Estimate for the inputs addressed here. The learning rule is a Hebbian type of synaptic plasticity combined with a plasticity for neuronal biases. Before providing the rule and the related proof, we will describe how our mixture model is equivalent to a WTA neural network.

\begin{algorithm}
	\caption{SoftHebb learning}\label{alg:SoftHebb}
	\begin{algorithmic}[1]
		\FORALL{neurons $k\in \{1,2,...,K\}$ in the layer,}
		\STATE initialize random weights and biases
		\ENDFOR
		\FORALL{training examples $\boldsymbol{x}$}
		\FORALL{neurons $k$}
		\STATE Calculate preactivation $u_k=\boldsymbol{w}_k\boldsymbol{x}$
		\ENDFOR
		\FORALL{neurons $k$}
		\STATE Optional: calculate activation $q'_k=h(u_k+w_{0k})$ \COMMENT{e.g. $h(x)=\exp(x)$}
		\STATE Calculate posterior (i.e.\ normalized activation) $y_k$ \COMMENT{e.g. Softmax}
		\ENDFOR
		\FORALL{neurons $k$}
		\FORALL{synapses $i$}
		\STATE calculate weight change $\Delta w_{ik}^{(SoftHebb)}=\eta \cdot y_k  \cdot \left(x_i-u_kw_{ik}\right)$
		\STATE update weight $w_{ik}\leftarrow w_{ik}+\Delta w_{ik}^{(SoftHebb)}$
		\ENDFOR
		\STATE calculate bias change $\Delta w_{0k}^{SoftHebb}=\eta e^{-w_{0k}}\left(y_k - e^{w_{0k}} \right)$
		\STATE update bias $w_{0k}\leftarrow w_{0k}+\Delta w_{0k}^{SoftHebb}$
		\ENDFOR
		\ENDFOR
	\end{algorithmic}
\end{algorithm}

\subsection{Equivalence of the probabilistic model to a WTA neural network}
\stepcounter{thm}
\label{sec:neuro_exp} The cosine similarity between the input vector and each centroid's parameters underpins the model (Eq.\ \ref{eq:multinomial}). This similarity is precisely computed by a linear neuron that receives normalized inputs
$\boldsymbol{x}^*\coloneqq\frac{\boldsymbol{x}}{||\boldsymbol{x}||}$ and normalizes its vector of synaptic weights:
$\boldsymbol{w}^*_k\coloneqq\frac{\boldsymbol{w}}{||\boldsymbol{w}||}$. Specifically, the neuron's summed weighted input
$u_k= \boldsymbol{w}^*_k\cdot\boldsymbol{x}^*$ then determines the cosine similarity of an input sample to the weight vector, thus computing the likelihood function of each component of the input mixture (Eq.\ \ref{eq:g_param}). It should be noted that even though $u_k$ depends on the weights of all input synapses, the weight values of other synapses do not need to be known to each updated synapse. Therefore, in the SoftHebb plasticity rule that we will present (Eq.\ \ref{eq:synplast}), the term $u_k$ is a local, postsynaptic variable that does not undermine the locality of the plasticity. The bias term of each neuron can store the parameter $w_{0k}$ of the prior $Q(C_k; w_{0k})$.
Based on these, it can also be shown that a set of $K$ such neurons can actually compute the Bayesian posterior, if the neurons are connected in a configuration that implements softmax. Softmax has a biologically-plausible implementation through lateral inhibition (divisive normalization) between neurons \citep{nessler2009stdp, nessler2013PLoS, moraitis2020shortterm}. Specifically, based on the model of Definition \ref{def:model}, the posterior probability is
\begin{equation}
	Q(C_k|\boldsymbol{x};\boldsymbol{w})=\frac{e^{u_k+w_{0k}}}{\sum_{l=1}^{K}e^{u_l+w_{0l}}}. \label{eq:Q_model}
\end{equation}
But in the neural description, $u_k+w_{0k}$ is the activation of the $k$-th linear neuron. That is, Eq.\ \ref{eq:Q_model} shows that the result of Bayesian inference of the hidden cause from the input $Q(C_k|\boldsymbol{x})$ is found by a softmax operation on the linear neural activations. In this equivalence, we will be using $y_k\coloneqq  Q(C_k|\boldsymbol{x};\boldsymbol{w})$ to symbolize the softmax output of the $k$-th neuron, i.e.\ the output after the WTA operation, interchangeably with $Q(C_k|\boldsymbol{x})$.
It can be seen in Eq.\ \ref{eq:Q_model} that the probabilistic model has one more, alternative, but equivalent neural interpretation. Specifically, $Q(C_k|\boldsymbol{x})$ can be described as the output of a neuron with exponential activation function (numerator in Eq.\ \ref{eq:Q_model}) that is normalized by its layer's total output (denominator). This is equally accurate, and more directly analogous to the biological description \citep{nessler2009stdp, nessler2013PLoS, moraitis2020shortterm}. This shows that the exponential activation of each individual neuron $k$ directly equals the $k$-th exponential component distribution of the generative mixture model (Eq.\ \ref{eq:multinomial}).
Therefore, the softmax-configured linear neurons, or equivalently, the normalized exponential neurons, fully implement the generative model of Definition \ref{def:model}, and also infer the Bayesian posterior probability given an input and the model parameters. However, the problem of calculating the model's parameters from data samples is a difficult one, if the input distribution's parameters are unknown. In the next sections we will show that this neural network can find these optimal parameters through Bayesian inference, in an unsupervised and on-line manner, based on only local Hebbian plasticity. It should be noted that the number of neurons $K$ must be chosen in advance, an aspect that may be regarded as human supervision.

\subsection{A Hebbian rule that optimizes the weights}
Several Hebbian-like rules exist and have been combined with WTA networks. For example, in the case of stochastic binary neurons and binary population-coded inputs, it has been shown that weight updates with an exponential weight-dependence find the optimal weights \citep{nessler2009stdp, nessler2013PLoS}. Oja's rule is another candidate \citep{oja1982simplified}. An individual linear neuron equipped with this learning rule finds the first principal component of the input data \citep{oja1982simplified}. A variation of Oja's rule combined with hard-WTA networks and additional mechanisms has achieved good experimental results performance on classification tasks \citep{krotov2019unsupervised}, but lacks the theoretical underpinning that we aim for. Here we propose a Hebbian-like rule for which we will show it optimizes the soft WTA's generative model. The rule is similar to Oja's rule, but considers, for each neuron $k$, both its linear weighted summation of the inputs $u_k$, and its nonlinear output of the WTA $y_k$:

\begin{equation}
	\label{eq:synplast}
	\boxed{\Delta w_{ik}^{(SoftHebb)}\coloneqq\eta \cdot y_k  \cdot \left(x_i-u_kw_{ik}\right),}
\end{equation}

where $w_{ik}$ is the synaptic weight from the $i$-th input to the $k$-th neuron, and $\eta$ is the learning rate hyperparameter. As can be seen, all involved variables are local to the synapse, i.e.\ only indices $i$ and $k$ are relevant. No signals from distant layers, from non-perisynaptic neurons, or from other synapses are involved.
By solving the equation $E[\Delta w_{ik}]=0$ where $E[\cdot]$ is the expected value over the input distribution, we can show that, with this rule, there exists a stable equilibrium value of the weights, and this equilibrium value is an optimal value according to Theorem \ref{thm:optimal}:
\begin{thm}
	\label{thm:equilib}
	The equilibrium weights of the SoftHebb synaptic plasticity rule are
	\begin{align}
		\begin{split}
			w_{ik}^{SoftHebb}&=c \cdot\mu_{p_k}(x_i)=\prescript{}{opt}w_{ik},\\ &\text{ where } c=\frac{1}{||\mu_{p_k}(\boldsymbol{x})||}.
		\end{split}
	\end{align}
\end{thm}
The proof is provided in the Methods section.
Therefore, our update rule (Eq.\ \ref{eq:synplast}) optimizes the neuronal weights.

Moreover, the following normalization Theorem is proven in the Methods section.
\begin{thm}
	\label{thm:SoftHebbnorm}
	The equilibrium weights of the SoftHebb synaptic plasticity rule of Eq.\ \ref{eq:synplast} are implicitly normalized by the rule to a vector of length 1.
\end{thm}
This then constrains the convergence of SoftHebb to a unique solution. Moreover, it can be used as a proxy for measuring the progress of convergence (see \cref{sec:exp_cross} and \cref{fig:performance}D).

\stepcounter{subsection}
\subsection{Local learning of neuronal biases as Bayesian priors}
For the complete optimization of the model, the neuronal biases $w_{0k}$ must also be optimized to satisfy Eq.\ \ref{eq:G0}, i.e.\ to optimize the Bayesian prior belief for the probability distribution over the $K$ input causes.
For the biases, we define the following rate-based rule inspired from the spike-based bias rule of \citep{nessler2013PLoS}:

\begin{equation}
	\boxed{\Delta w_{0k}^{SoftHebb}=\eta e^{-w_{0k}}\left(y_k - e^{w_{0k}} \right).}
\end{equation}

With the same technique we used for Theorem \ref{thm:equilib}, we also provide proof in the Methods that the equilibrium of the bias with this rule matches the optimal value $\prescript{}{opt}{}w_{0k}=\ln P(C_k)$ of Theorem \ref{thm:optimal}:
\begin{thm}
	\label{thm:nesslerbias}
	The equilibrium biases of the SoftHebb bias learning rule are
	\begin{gather}
		w_{0k}^{SoftHebb}=\ln P(C_k)=\prescript{}{opt}{}w_{0k}.
	\end{gather}
\end{thm}

\subsection{Alternate activation functions. Relation to Hard WTA}
\label{sec:activation_fn}
The model of Definition \ref{def:model} uses for each component $p(\boldsymbol{x}|C_k)$ an exponential probability distribution with a base of Euler's e, equivalent to a model using similarly exponential neurons (\cref{sec:neuro_exp}). Depending on the task, different probability distribution shapes, i.e.\ different neuronal activation functions, may be better models. This is compatible with our theory (see \cref{app:alternate_act}). Firstly, the base of the exponential activation function can be chosen differently, resulting in a softmax function with a different base, such that Eq.\ \ref{eq:Q_model} becomes more generally
\begin{equation}
Q^{(alt)}(C_k|\boldsymbol{x})=\frac{b^{u_k+w_{0k}}}{\sum_{l=1}^{K}b^{u_l+w_{0l}}}. \label{eq:Qb_model}
\end{equation}
This is equivalent to Temperature Scaling \citep{hinton2015distilling}, a mechanism that also maintains the probabilistic interpretation of the softmax output. The alternate-base version can also be implemented by a normalized layer of exponential neurons, which are compatible with our theoretical derivations and the optimization by the plasticity rule of Eq.\ \ref{eq:synplast}.
Interestingly, this integrates the hard WTA into the SoftHebb framework. Specifically, hard WTA is a special case of SoftHebb with an infinite base $b$ underlying the softmax, or equivalently a temperature of zero. Therefore, a Hebbian hard WTA, if used with the plasticity rule that we derived, is expected to show certain similarity to the soft WTA implementation. However, in the Experimental Results section we show that the soft version does have advantages. Namely, it leads to higher classification accuracy (\cref{fig:performance}a,b), it converges faster (\cref{fig:performance}b,c,d), and -- by allowing the network's interpretation as a Bayesian mixture model -- it enables SoftHebb's treatment as a generative model that can be sampled from and can generate synthetic objects \cref{fig:gan}.
This understanding and comparison is important because hard WTA is often chosen to underlie Hebbian learning \cite{amato2019hebbian, grinberg2019local, krotov2019unsupervised, lagani2021hebbian}.

Moreover, we show in the Methods section that other activation functions than exponential or softmax are also supported without a. soft WTA models can be constructed by rectified linear units (ReLU) or in general by neurons with any non-negative monotonically increasing activation function, and their weights are also optimized by the same plasticity rule.

\subsection{Cross-entropy minimization without supervision}
\paragraph*{SoftHebb as a discriminator.} Even though SoftHebb's soft WTA is a generative and not a discriminative model, i.e.\ it models the distribution $p(\boldsymbol{x})$ as $q(\boldsymbol{x};\boldsymbol{w})$, it can also be used for discrimination of the input classes $C_k$, i.e.\ classification, using Bayes' theorem:
\begin{equation}
Q(C_k|\boldsymbol{x};\boldsymbol{w})= \frac{q(\boldsymbol{x}|C_k;\boldsymbol{w})}{q(\boldsymbol{x};\boldsymbol{w})}.
\end{equation}
\paragraph*{SoftHebb minimizes cross-entropy of the true causes.}
It can be shown that while the generative model is optimized by SoftHebb, its discriminative aspect is also optimized. Specifically, the algorithm minimizes in expectation $H^{C}_Q\coloneqq H(P(C|\boldsymbol{x}), Q(C|\boldsymbol{x}))$,  i.e.\ the cross-entropy of the causes $Q(C_k|\boldsymbol{x})$ that it infers, from the true causes of the data $P(C_k|\boldsymbol{x})$:
\begin{equation}
	\boldsymbol{w}^{SoftHebb}= arg \min_{\boldsymbol{w}} H^{C}_Q.
	\label{eq:xent_c}
\end{equation}
That corollary follows from the proof of Theorem \ref{thm:optimal}. The Theorem's proof involved showing that SoftHebb minimizes the KL divergence $D_{KL}(p(\boldsymbol{x})||q(\boldsymbol{x};\boldsymbol{w}))$ of the model $q(\boldsymbol{x};\boldsymbol{w})$ from the data $p(\boldsymbol{x})$. But KL divergence is the difference of cross-entropy $H(p(\boldsymbol{x}), q(\boldsymbol{x};\boldsymbol{w}))$ minus the entropy $S_p$ of the data:
$D_{KL}(p(\boldsymbol{x}))||q(\boldsymbol{x};\boldsymbol{w}))=H(p(\boldsymbol{x}), q(\boldsymbol{x};\boldsymbol{w}))-S_p$.

 Therefore, since the entropy $S_p$ of the data distribution does not depend on the optimized model parameters, we conclude that minimizing KL divergence through SoftHebb implies also minimizing cross-entropy $H(p(\boldsymbol{x}), q(\boldsymbol{x};\boldsymbol{w}))$.

\paragraph*{Disparity between direct cause and label.}
\label{sec:cause_label}
So far we have shown that beyond the generative model, SoftHebb also optimizes the discriminative ability of the Bayesian model, but specifically pertaining to discriminating among the direct causes of the data.
However, ML practice is frequently interested in categorizations that may not correspond to the true causes that directly relate to the data.
In other words, in tasks of classification into a set of class labels, the label set -- let us denote this by $L$ -- chosen by a human supervisor may not correspond exactly to the true and single cause $C$ that generates the datapoints, which is what SoftHebb's unsupervised process learns to infer. This difference is depicted in \cref{fig:causal}(a).
Nevertheless, SoftHebb's process does minimize cross-entropy with respect to the labels too, as long as the label set is reasonable -- which we will now formalize.
For example, consider the commonly used benchmark dataset MNIST. The 10 labels indicating the 10 decimal digits do not correspond exactly to the true cause of each example image. In reality, the direct cause $C$ generating each MNIST example in the sense implied by causal inference is not the digit cause on its own, which corresponds to the MNIST label, but rather it is a combination of the digit $L$ with one of many handwriting styles $S$.
That is, the probabilistic model is such that
the direct cause $C$ of each sample is dual, i.e.\ there exists a digit $L_l\,(l\in \{0,1,..., 9\})$ and a style $S_s$ that jointly compose the direct cause (see also \cref{fig:causal}):
\begin{equation}
	\label{eq:causes0}P(C_k)\coloneqq P(C=C_k)=P(L_l) P(S_s).
\end{equation}

\begin{figure}[h]
	\centering
	\captionsetup[subfigure]{justification=centering}
	\begin{subfigure}[t]{0.3\textwidth}
		\begin{tikzpicture}
			\node[state] (x) {$\boldsymbol{x}$};
			\node (caption) [below =of x,yshift=1.5cm,xshift=-1.5cm] {(a)};
			\node[state, blue] (c) [above =of x] {$C$};
			\node[state, red] (l) [above left =of c] {$L$};
			\node[state] (s) [above right =of c] {$S$};
			\path [red] (l) edge (c);
			\path (s) edge (c);
			\path [blue] (c) edge (x);
		\end{tikzpicture}
		\label{fig:causal_a}
	\end{subfigure}
	\hspace{1cm}
	\begin{subfigure}[t]{0.3\textwidth}
		\begin{tikzpicture}
			\node (x) {Image};
			\node (caption) [below =of x,yshift=1.5cm,xshift=-2cm] {(b)};
			\node [blue] (c) [above =of x,yshift=-0.2cm,text width=4cm, align=center] {True direct cause \\ (Inferred by SoftHebb)};
			\node [red] (l) [above left =of c,yshift=-0.3cm,xshift=3.4cm,text width=3cm, align=center] {MNIST label\\(Digit 0-9)};
			\node (s) [above right =of c,yshift=-0.3cm,xshift=-3.4cm,text width=3cm, align=center] {Handwriting\\style};
			\path [red] (l) edge (c);
			\path (s) edge (c);
			\path [blue] (c) edge (x);
		\end{tikzpicture}
		\label{fig:causal_b}
	\end{subfigure}
	\caption{(a) A causal graph where the single direct hidden cause $C$ that generates the observed data $\boldsymbol{x}$ is itself affected by two root causes $L$ and $S$. (b) The same graph annotated, with root causes as they correspond to the MNIST dataset.}
	\label{fig:causal}
\end{figure}
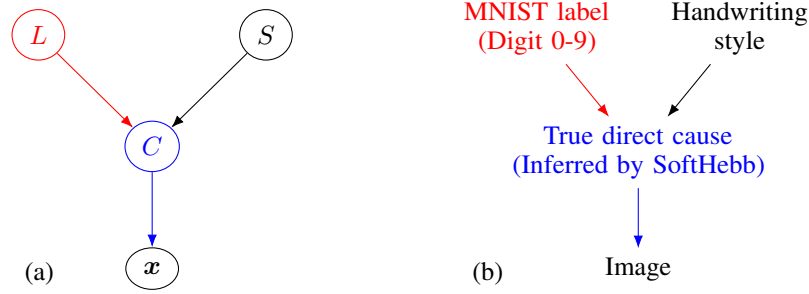

\paragraph*{SoftHebb minimizes cross-entropy of the labels.}
This relationship between the labels $L$ and the direct causes $C$ (see \cref{fig:causal}(a), red arrow) can be written as
\begin{equation}
	\label{eq:causes}
	P(L_l)=\sum_k P(L_l|C_k) P(C_k).
\end{equation}

Therefore, SoftHebb's $Q(C)$ implicitly defines a model $Q(L)$:
\begin{equation}
	Q(L_l)\coloneqq\sum_{k} P(L_l|C_k)Q(C_k). \label{eq:causesQ}
\end{equation}

In Eqs. \ref{eq:causes} and \ref{eq:causesQ}, the term $P(L_l|C_k)$ is fixed by the data-generation process.
As a consequence, to minimize cross-entropy between $P(L)$ and $Q(L)$ is to minimize cross-entropy between $P(C)$ and $Q(C)$, which SoftHebb does (see \cref{eq:xent_c}).

Therefore, SoftHebb minimizes cross-entropy $H^{L}_Q$ between its implicit model of the labels and the true labels $L$:
\begin{equation}
	\boxed{\text{\textbf{Cross-entropy minimization: }}\boldsymbol{w}^{SoftHebb}= arg \min_{\boldsymbol{w}} H^{C}_Q= arg \min_{\boldsymbol{w}} H^{L}_Q.} \label{eq:argminHQ}
\end{equation}
This is remarkable, given that SoftHebb never accesses the labels or any other supervisory signal.

\paragraph*{Measuring loss minimization: Post-hoc cross-entropy}
\label{sec:causes}
It has not been obvious how to correctly measure the loss of an unsupervised WTA network during the learning process, since the ground truth for \textit{causes} $C$ is often not available, even if \textit{labels} $L$ are available, as we reasoned above. Using our theoretical result about cross-entropy minimization, here we propose a method for measuring the loss minimization during training in spite of this cause-label mismatch. To obtain $Q(L_l|\boldsymbol{x})$ of \cref{eq:causesQ} and measure the cross-entropy, the causal structure $P(L_l|C)$ (see \cref{fig:causal}, red) is missing, but it can be represented by a supervised classifier $Q_2(L_l|Q(C|\boldsymbol{x}))$ of SoftHebb's outputs, trained using the labels $L_l$.
Therefore, by (a) unsupervised training of SoftHebb, then (b) training a supervised classifier on top, and finally (c1) repeating the training of SoftHebb with the same initial weights and ordering of the training inputs as in step (a), while (c2) measuring the trained classifier's loss, we can observe the cross-entropy loss $H^{L}$ of SoftHebb while it is being minimized, and infer that $H^{C}$ is also minimized (Eq.\ \ref{eq:argminHQ}). We call this the \textit{post-hoc cross-entropy} method, and it enables evaluation of the learning process in a theoretically sound manner during experimentation (see \cref{sec:exp_cross} and \cref{fig:performance}C).

\section{Experimental Results}
We implemented the theoretical SoftHebb model in simulations and tested it in the task of learning to classify MNIST handwritten digits. The network received the MNIST frames normalized by their Euclidean norm, while the plasticity rule that we derived updated its weights and biases in an unsupervised manner. We used $K=2000$ neurons. First we trained the network for 100 epochs, i.e.\ randomly ordered presentations of the 60000 training digits.
Each training experiment was repeated five times with varying random initializations and input order. We report the mean and standard deviation of the resulting accuracies.
Inference of the input labels by the WTA network of 2000 neurons was performed in two different ways. The first approach is single-layer, where, after training the network, we assigned a label to each of the 2000 neurons, in a standard approach that is used in unsupervised clustering. Namely, for each neuron, we found the label of the training set that makes it win the WTA competition most often. In this single-layer approach, this is the only time when labels were used, and at no point were weights updated using labels.
The second approach was two-layer and based on supervised training of a perceptron, i.e.\ linear classifier on top of the WTA layer. The classifier layer was trained with the Adam optimizer \citep{kingma2015adam} and cross-entropy loss for 100 epochs, while the previously-trained WTA parameters were frozen.

\subsection{Indicative accuracies in standard setting}
\label{sec:vshardWTA}
\begin{figure*}[ht]
	\centering
	\includegraphics[width = 140mm]{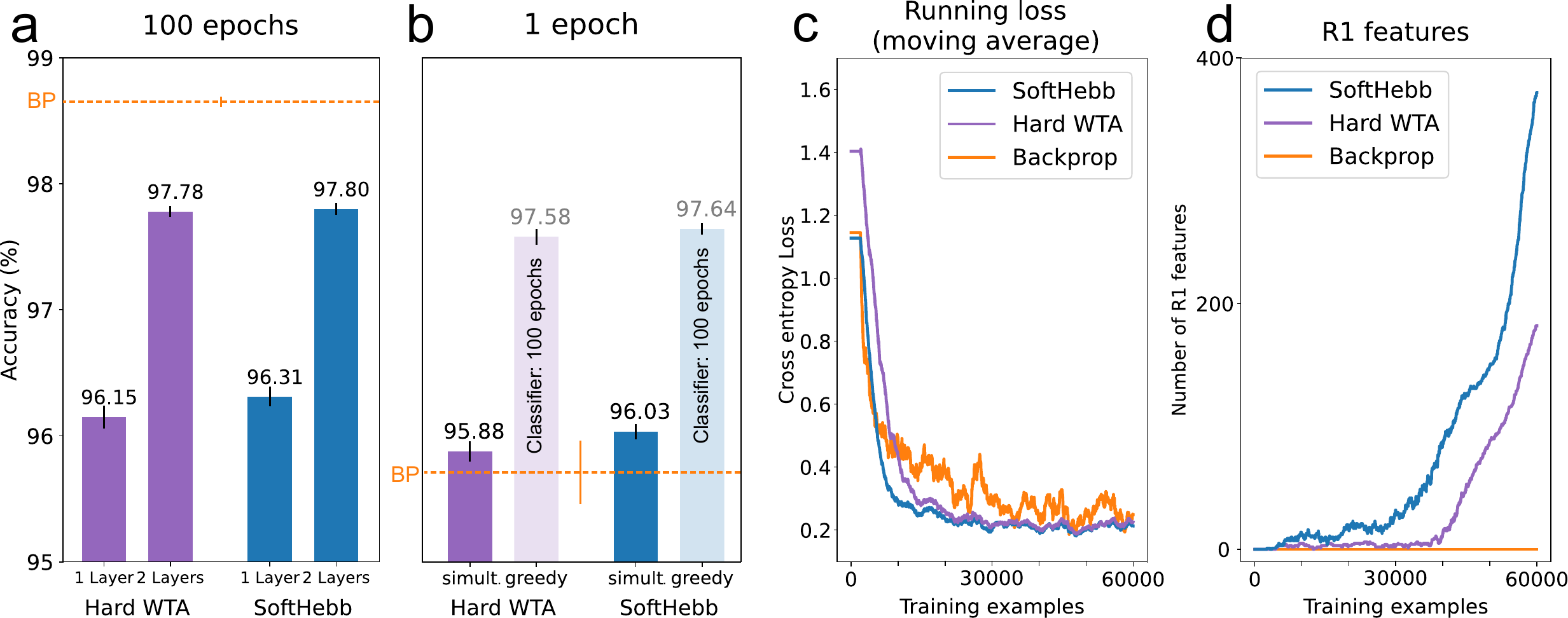}
	\caption{Performance of SoftHebb on MNIST compared to hard WTA and backpropagation. (a) SoftHebb with a finite softmax base slightly outperforms its hard-WTA special case, especially in the single-layer evaluation. (b) SoftHebb learns fast, performing almost as well in the first training epoch as later, and even outperforms end-to-end backpropagation (BP, horizontal line). (c) SoftHebb minimizes the post-hoc cross-entropy loss, as predicted by the theoretical results. In addition, SoftHebb minimizes it faster than its hard-WTA version, and faster than supervised backpropagation of the loss. (d) SoftHebb learns weight vectors that converge to a sphere of radius 1 (``R1 features''). The soft version is faster to converge also under this metric that evaluates the learned representation itself.}
	\label{fig:performance}
\end{figure*}

SoftHebb achieved an accuracy of $(96.31\pm0.06)\%$ and $(97.80\pm0.02)\%$ in its 1- and 2-layer form respectively.
To test the strengths of the soft-WTA approach combined with training the priors through biases, which makes the network Bayesian, we also trained the weights of a hard-WTA network, i.e.\ a model equivalent to SoftHebb with an infinite base in the softmax.
The SoftHebb model slightly outperformed the special case of hard WTA (Fig. \ref{fig:performance}a), especially in the 1-layer case where the supervised 2nd layer cannot compensate for the drop in accuracy. Speed comparisons reveal that this is due to a faster convergence by SoftHebb in terms of learning examples (see \cref{sec:exp_cross} and \cref{fig:performance}d).
As an indicative baseline, we also trained a multi-layer perceptron (MLP) with one hidden layer of also 2000 neurons, exhaustively with end-to-end backpropagation and tested it. This was expected to perform significantly better and indeed it reached an accuracy of $(98.65\pm0.06)\%$ (\cref{fig:performance}a, horizontal dashed line). This is not surprising, due to end-to-end training, supervision, and the MLP being a discriminative model as opposed to a generative model merely applied to a classification task, as SoftHebb is. If the Bayesian and generative aspects that follow from our theory were not required, several mechanisms exist to enhance the discriminative power of WTA networks \citep{krotov2019unsupervised}, and even an untrained, random projection layer in place of a trained WTA performs well \citep{illing2019biologically}. Our approach however does have surprising advantages even in a discriminative task, and we report these in the next sections.

\subsection{Speed advantages of SoftHebb. Cross-entropy minimization}
\label{sec:exp_cross}
\paragraph*{Learning speed: SoftHebb outperforms hard WTA and backpropagation in the first epoch.}
Next, we evaluated SoftHebb's data-efficiency and speed by comparing it to other models during the first training epoch.
In the common, ``greedy'' training of such networks, layer L+1 is trained only after layer L is trained on the full dataset and its weights frozen. We trained in this manner a second layer as a supervised classifier for 100 epochs after a single unsupervised learning epoch in the first layer. SoftHebb with a non-zero temperature again slightly outperformed its Hard WTA version, showing it extracts superior features (Fig. \ref{fig:performance}b, light-coloured bars).
More importantly, we tested a truly single-epoch scenario, without longer training for the second layer. SoftHebb further outperforms Hard WTA (Fig. \ref{fig:performance}b, dark coloured bars). Strikingly, it even outperforms end-to-end (e2e) BP in 1-epoch accuracy. For the 1-epoch experiments, SoftHebb was trained in the fully on-line setting, where each iteration includes a single training example, i.e.\ a ``mini-batch'' of size 1, whereas for backpropagation we tuned the batch size and learning rate to the single-epoch setting and found that stochastic gradient descent with a batch size of 4 was best (see \ref{sec:hyperparams}), which we used.

\paragraph*{SoftHebb minimizes cross-entropy, and faster than backpropagation.}
To empirically confirm the theoretical result about unsupervised minimization of cross-entropy by SoftHebb, we measured the post-hoc cross-entropy loss as derived in \cref{sec:causes}.
The resulting training curve is depicted in \cref{fig:performance}c, for a representative example of a training run. The first observation is that the experiment confirms our theoretical prediction that SoftHebb minimizes the cross-entropy between the model and the labels, which is remarkable, considering that the labels were hidden from the model.
In comparison to the baselines, SoftHebb is faster to converge than its hard-WTA special case and than backpropagation, as measured by number of training iterations/examples. All three algorithms were trained in the fully on-line setting. That is consistent with the accuracy results for the first epoch (\cref{sec:vshardWTA}). This specific manifestation of SoftHebb's speed advantage is particularly interesting, as it concerns a loss function that backprop minimizes explicitly and has access to it. One might have intuitively expected then that backprop's supervised minimization would be faster. SoftHebb's speed advantage might be attributed to its Bayesian nature, which takes account of each datapoint of evidence optimally.

\subsection{SoftHebb improves representation learning over hard WTA}
As a further insight into the differences between SoftHebb and hard-WTA learning, we measured throughout learning the number of learned features that lie on a hypersphere with a radius of $1\pm0.01$ (R1 features), according to their Euclidean norm. The SoftHebb learning algorithm converges to such a normalization in theory
(end of Appendix \ref{sec:proofs}, Theorem \ref{thm:SoftHebbnorm}), and Fig. \ref{fig:performance}D validates that it does, but also that it does so faster than its hard-WTA special case. This demonstrates SoftHebb's superiority in unsupervised representation learning and speed, from the perspective of its convergence to the optimal generative Bayesian model, rather than from its discriminative ability.

\subsection{Update unlocking}
In the simultaneous experiments, the Hebbian networks has an additional important advantage. By using the delta rule for the 2nd layer, each individual training example updates both layers. In contrast to end-to-end (e2e) backpropagation, this simultaneous method does not suffer from the update-locking problem \citep{czarnecki2017understanding, frenkel2021learning}, i.e.\ the first layer can learn from the next example before the current input is even processed by the higher layer, let alone backpropagated. This is a consequence of the locality of the plasticity, and solves an important inefficiency and implausibility of backpropagation and most of its approximations.

Hebbian learning compared to backpropagation has not generally been considered superior for its accuracy, but for other potential benefits. Here we show that, for small problems demanding fast learning, SoftHebb may be superior to backpropagation even in terms of accuracy, in addition to its biological plausibility and efficiency.

\subsection{Robustness to noise and adversarial attacks}
\label{sec:robustness}
\paragraph*{Robustness comparison with backpropagation.}
\begin{figure*}[ht]
	\centering
	\includegraphics[width = 140mm]{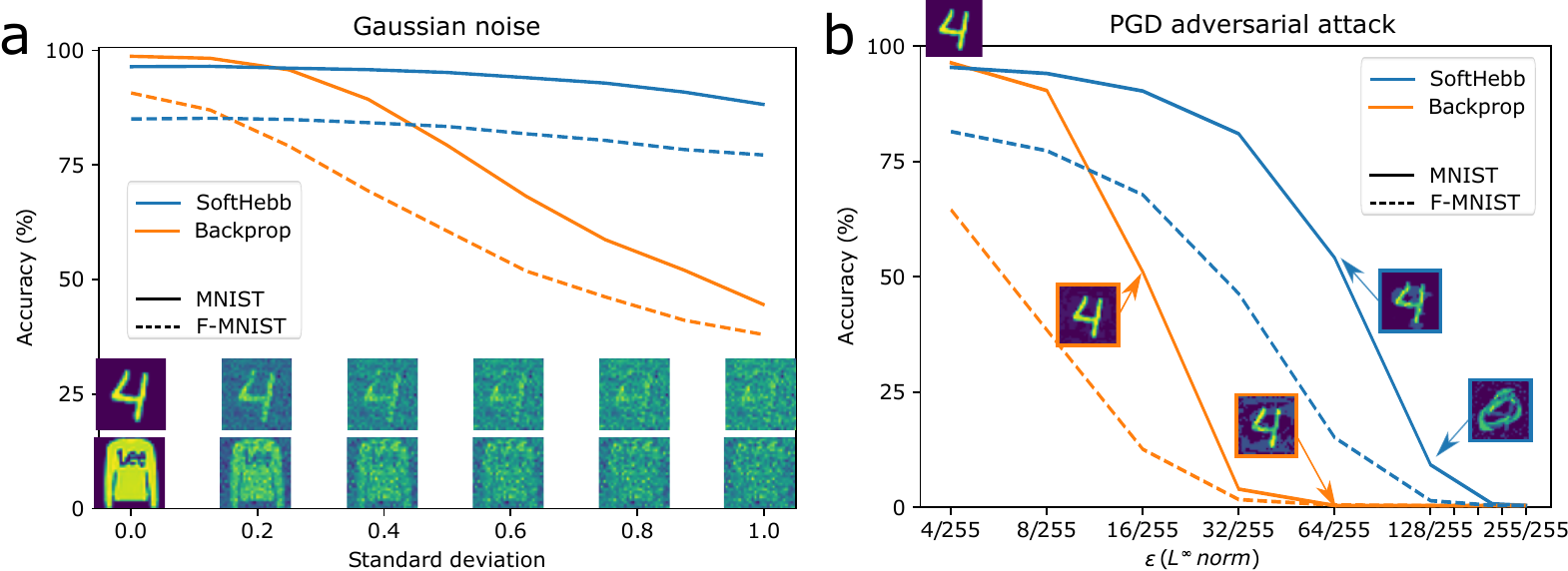}
	\caption{Noise and adversarial attack robustness of SoftHebb and of backpropagation-trained MLP on MNIST and Fashion-MNIST. The insets show one example from the testing set and its perturbed versions, for increasing perturbations. (a) SoftHebb is highly robust to noise, even in very noisy settings, in contrast to backprop. (b) MLP's MNIST accuracy drops to \char`~50\% by hardly perceptible perturbations ($\epsilon=16/255$), while SoftHebb requires visually noticeable perturbations ($\epsilon=64/255$) for similar drop in performance. At that degree of perturbation, the MLP has already dropped to zero. SoftHebb deflects the attack: it forces the attacker to produce examples of truly different classes - the original digit ``4'' is perturbed to look like a ``0'' (see also Fig. \ref{fig:gan}). The attack of the backprop-trained network does not confuse a human observer even at $\epsilon=64/255$.}
	\label{fig:robustness}
\end{figure*}
Based on the Bayesian, generative, and purely input-driven learning nature of the algorithm (as opposed to driven by top-down signals), we hypothesized that SoftHebb may be more robust to input perturbations. Indeed, we tested the trained SoftHebb and MLP models for robustness, and found that SoftHebb is significantly more robust than the backprop-trained MLP, both to added Gaussian noise and to projected gradient descent (PGD) adversarial attacks (see Fig. \ref{fig:robustness}). PGD \citep{madry2017towards} produces perturbations in a direction that ascends the loss of each targeted network, and in size controlled by a parameter $\epsilon$. It is a white-box attack, that has access to the values of the weights, and is considered one of the strongest adversarial attacks.
Strikingly, the Hebbian WTA model has a visible tendency to deflect the attacks, i.e.\ its most confusing examples actually belong to a perceptually different class (Fig. \ref{fig:robustness}B and \cref{fig:gan}). This effectively nullifies the attack and was previously shown in elaborate SOTA adversarial-defence models \citep{qin2020deflecting}. The attack's parameters were tuned systematically (see \cref{app:adversarial}). 

\paragraph*{Robustness comparison with K-means and PCA.}
\begin{figure}[h]
	\centering
	\includegraphics[width = 140mm]{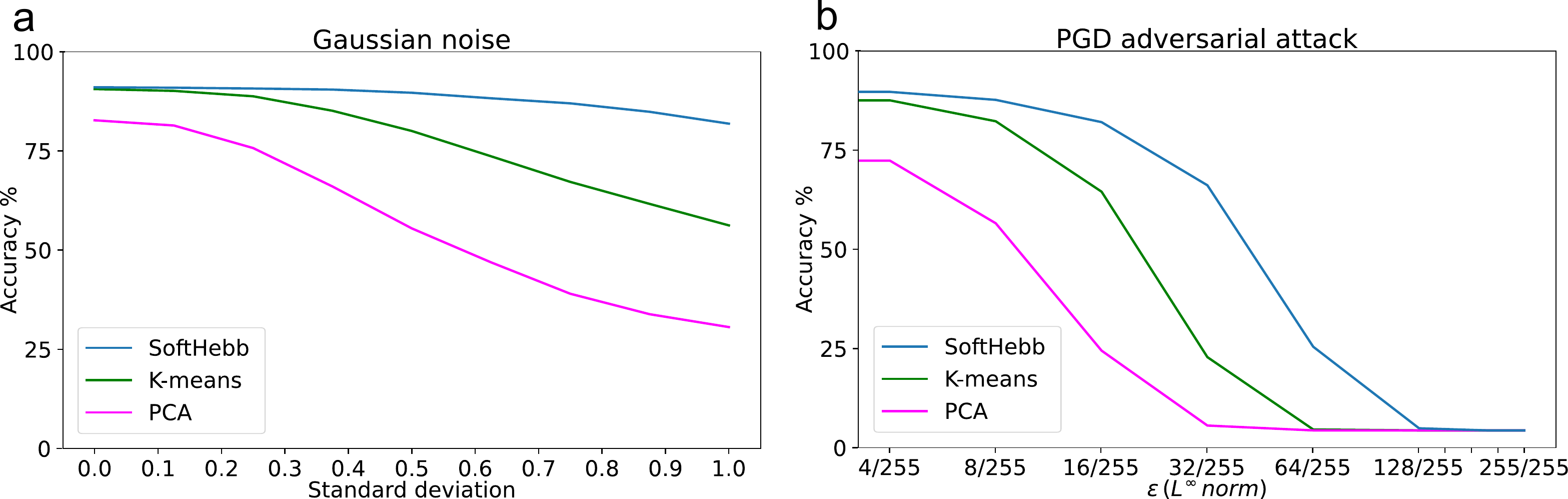}
	\caption{Unsupervised algorithms: Noise and adversarial attack robustness of SoftHebb, K-means, and PCA. SoftHebb is the most robust.}
	\label{fig:kmeans}
\end{figure}
It is possible that the observed robustness is not unique to SoftHebb and can be reproduced by other unsupervised learning rules. To test this possibility, we compared SoftHebb with principal component analysis (PCA) and K-means. We used 100 neurons, principal components, or centroids respectively. In PCA and K-means we then treated the learned coefficients as weight vectors of neurons and applied an activation function, i.e.\ non-linearity, to then train a supervised classifier on top. First we attempted softmax as in SoftHebb. However the unperturbed test accuracy achieved at convergence was much lower. For example, on MNIST, K-means only reached an accuracy of $53.64\%$ and PCA $28.55\%$, whereas SoftHebb reached $91.06\%$. Therefore, we performed the experiment again, but with ReLU activation for K-means and PCA, reaching  $90.61\%$ and $82.74\%$ respectively. Then we tested for robustness, revealing that SoftHebb's learned features are in fact more robust than those of the other two unsupervised algorithms (Fig. \ref{fig:kmeans}).
For K-means, the centroids were initialized at positions equal to K, i.e.\ 100 randomly sampled datapoints from the MNIST training set. We also experimented with random initialization from a uniform distribution, which did not produce a significantly different behaviour.

\paragraph*{Effect of softmax on adversarial robustness.}
It is possible that the observed robustness of SoftHebb is due to the use of softmax as an activation function. To test this, we compared the SoftHebb network from Fig. \ref{fig:robustness} with a same-size backprop-trained 2-layer network, but in this case the hidden layer's summed weighted input was passed through a softmax instead of ReLU before forwarding to the 2nd layer. First, we observed that at convergence, the backprop-trained network did not achieve SoftHebb's accuracy on either MNIST ($94.38\%$) or Fashion-MNIST ($75.90\%$). Increasing the training time to 300 epochs did not help. Second, as can be seen in \ref{fig:softmax_attack}, backpropagation remains significantly less robust than SoftHebb to the input perturbations. This, together with the previous control experiments, suggests that, rather than its activation function, it is SoftHebb's learned representations that are responsible for the network's robustness.

\begin{figure}[h]
	\centering
	\includegraphics[width = 140mm]{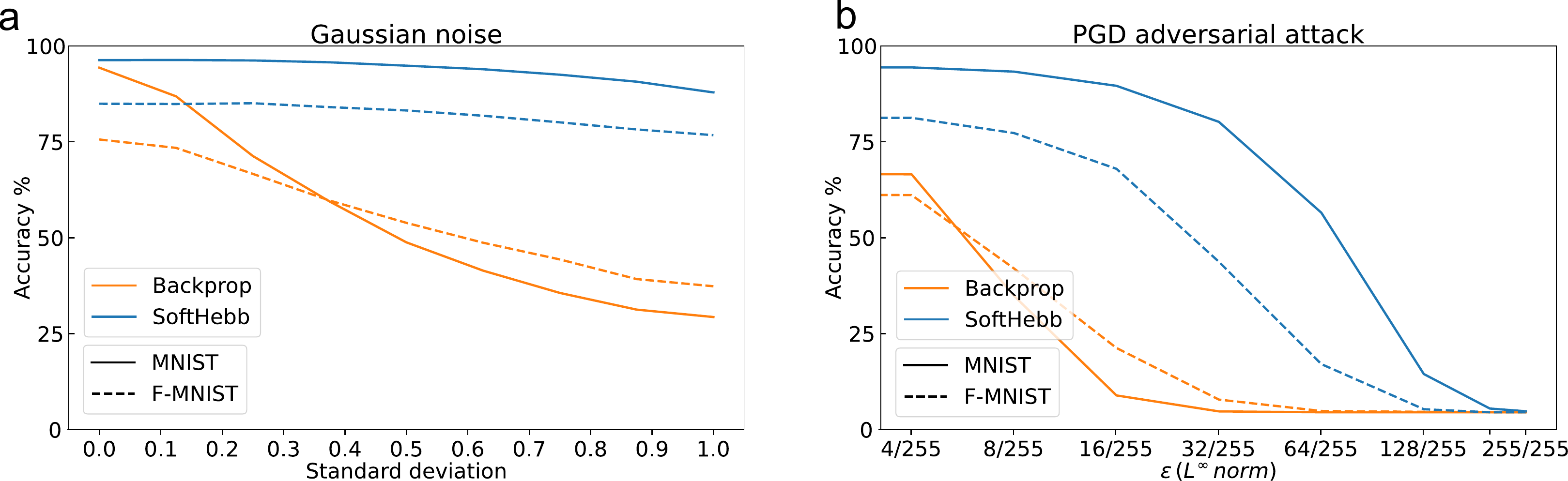}
	\caption{Softmax-based networks: Noise and adversarial attack robustness of SoftHebb and of backpropagation-trained softmax-MLP on MNIST and Fashion-MNIST. Both SoftHebb and the MLP use a softmax activation at the hidden layer. Backpropagation remains less robust than SoftHebb.}
	\label{fig:softmax_attack}
\end{figure}
\subsection{SoftHebb's generative adversarial properties}
\begin{figure*}[h]
	\centering
	\includegraphics[width = 135mm]{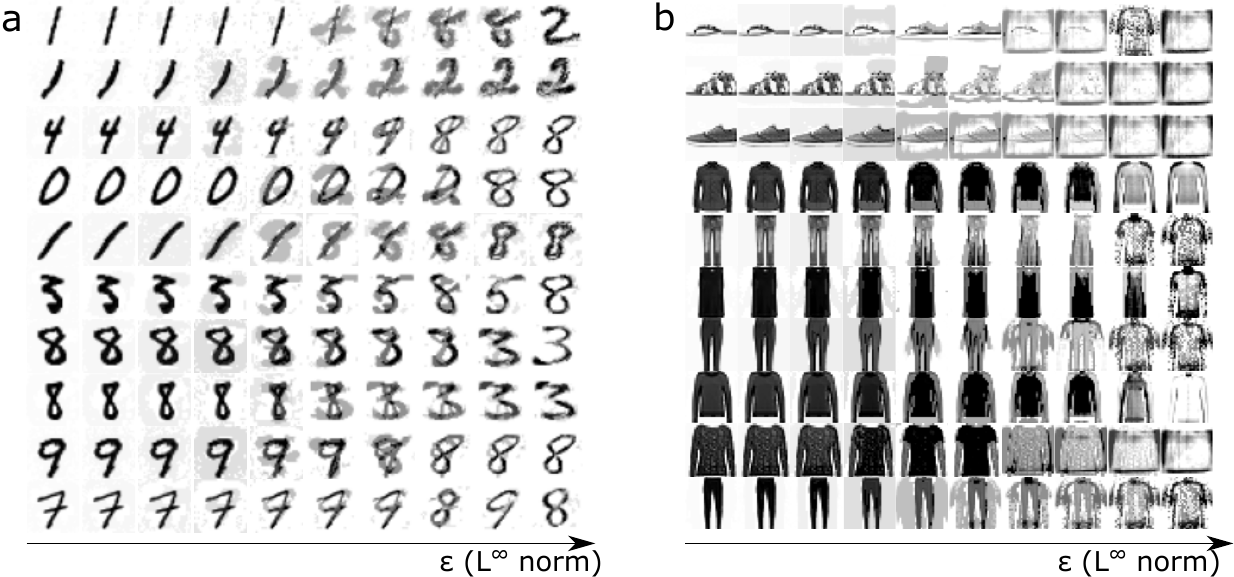}
	\caption{Synthetic objects generated by the adversarial pair PGD attacker/SoftHebb model for (a) MNIST and (b) F-MNIST. SoftHebb's inherent tendency to deflect the attack is visible, i.e.\ the strongest perturbations truly belong to different classes. Generation of synthetic objects that are interpolations between different classes of the true data distribution can also be seen. This generative property was previously unknown for such simple networks.}
	\label{fig:gan}
\end{figure*}
The pair of the adversarial attacker with the generative SoftHebb model essentially composes a generative adversarial network (GAN), even though the term is usually reserved for pairs \textit{trained }in tandem \citep{goodfellow2014generative, creswell2018generative}. As a result, the model could inherit certain properties of GANs. It can be seen that it is able to generate interpolations between input classes (Fig. \ref{fig:gan}). The parameter $\epsilon$ of the adversarial attack can control the balance between the interpolated objects. Similar functionality has existed in the realm of GANs \citep{radford2015unsupervised}, autoencoders \citep{berthelot2018understanding}, and other deep neural networks \citep{bojanowski2017optimizing}, but was not known for simple biologically-plausible models.

\subsection{Extensibility of SoftHebb: F-MNIST, CIFAR-10}
Finally, we performed preliminary tests on two more difficult datasets, namely Fashion-MNIST \citep{xiao2017/online}, which contains grey-scale images of clothing products, and CIFAR-10 \citep{krizhevsky2009learning}, which contains RGB images of animals and vehicles. We did not tune the Hebbian networks' hyper-parameters extensively, so accuracies on these tasks are not definitive but do give a good indication. Future experiments could use for example a recent Bayesian hyperparameter-optimization scheme \citep{cowen2022hebo}. On F-MNIST, the SoftHebb model achieved a top accuracy of $87.46\%$ whereas a hard WTA reached a similar accuracy of $87.49\%$. A supervised MLP of the same size achieved a test accuracy of $90.55\%$.
SoftHebb's generative interpolations (Fig. \ref{fig:gan}B) are reconfirmed on the F-MNIST dataset, as is its robustness to attacks, whereas, with very small adversarial perturbations, the MLP drops to an accuracy lower than the SoftHebb model (dashed lines in Fig. \ref{fig:robustness}). 	
On CIFAR-10's preliminary results, the hard WTA and SoftHebb achieved an accuracy of $49.78\%$ and $50.27\%$ respectively. In every tested dataset, it became clear that SoftHebb learns in fewer iterations than either backpropagation or a hard WTA, by observing the loss and the learned features as in Fig. \ref{fig:performance}c \& d.
The fully-connected SoftHebb layer is applicable to MNIST because the data classes are well-clustered directly in the feature-space of pixels. That is, SoftHebb's probabilistic model's assumptions (Def. \ref{def:data}) are quite valid for this feature space, and increasing the number of neurons for discovering more refined sub-clusters does help. However, for more complex datasets, this approach alone has diminishing returns and multilayer networks will be needed.
Towards this, we believe that a convolutional version of SoftHebb will be a key aspect to enable a distributed feature-representation despite the concentrated activation of WTA networks.
However, we leave multilayer networks and further experiments for future work \citep{journe2022hebbian}, and in the present article we focus on the theoretical foundation and properties of the individual SoftHebb layer.

\section{Methods}
\subsection{Proofs of theoretical results}
\label{sec:proofs}

\begin{proof}[\textbf{Proof of Theorem \ref{thm:optimal}}]
	A shorter version of this Theorem's proof was provided in the supplementary material of \citet{moraitis2020shortterm}, but for completeness we also provide it here.
	The parameters of model $q$ are optimal $\boldsymbol{w}=\prescript{}{opt}{}\boldsymbol{w}$ if they   minimize the model's Kullback-Leibler divergence with the data distribution $p$. $D_{KL}(p(\boldsymbol{x})||q(\boldsymbol{x};\boldsymbol{w}))$.
	Because $p_k\coloneqq p(\boldsymbol{x}|C_k)$ is independent from $p_l$, and $q_k\coloneqq q(\boldsymbol{x}|C_k; \boldsymbol{w}_k)$ is independent from $ \boldsymbol{w}_l$ for every $l\neq k$,
	we can find the set of parameters that minimize the KL divergence of the mixtures, by minimizing the KL divergence of each component $k$:
	$\min D_{KL}(p_k||q_k),\, \forall k,
	$ and simultaneously setting
	\begin{equation}
		P(C_k)=Q(C_k;w_{0k}), \, \forall k.
		\label{eq:priorcondition}
	\end{equation}		
	From Eq.\ \ref{eq:g0_param} and this last condition, Eq.\ \ref{eq:G0} of the Theorem is proven: \begin{equation}
		\prescript{}{opt}{}w_{0k}=\ln P(C_k). \nonumber
	\end{equation}
	
	Further,
	\begin{gather}
		\prescript{}{opt}{}\boldsymbol{w}_k  \coloneqq \arg \min_{ \boldsymbol{w}_k } D_{KL}(p_k||q_k)\nonumber\\
		=\arg \min_{ \boldsymbol{w}_k }\int_{\boldsymbol{x}}  p_k  \ln \frac{p_k }{q_k }d\boldsymbol{x}\nonumber\\
		=\arg \min_{ \boldsymbol{w}_k }\int_{\boldsymbol{x}}  p_k  \ln p_k  - p_k  \ln q_k  d\boldsymbol{x}\nonumber\\
		=\arg \min_{ \boldsymbol{w}_k }\int_{\boldsymbol{x}}  - p_k  \ln q_k  d\boldsymbol{x} \label{eq:entropy_pt2}\\
		=\arg \max_{ \boldsymbol{w}_k }\int_{\boldsymbol{x}}  p_k  \ln q_k  d\boldsymbol{x}\label{eq:ln_q}\nonumber \\
		=\arg \max_{ \boldsymbol{w}_k }\int_{\boldsymbol{x}}  p_k  \ln e^{u_k}  d\boldsymbol{x}  \label{eq:cosine} \\
		=\arg \max_{ \boldsymbol{w}_k }\int_{\boldsymbol{x}}  p_k  u_k  d\boldsymbol{x} \nonumber\\
		=\arg \max_{ \boldsymbol{w}_k } \mu_{p_k}\left(u_k\right) \nonumber\\
		=\arg \max_{ \boldsymbol{w}_k } \mu_{p_k}\left(\cos\left( \boldsymbol{w}_k, \boldsymbol{x}\right)\right).
		\label{eq:argmin}
	\end{gather}
	where we used for Eq.\ \ref{eq:entropy_pt2} the fact that $\int_{\boldsymbol{x}}  p_k  \ln p_k d\boldsymbol{x}$ is a constant because it is determined by the environment's data and not by the model's parametrization on $\boldsymbol{w}$.
	Eq.\ \ref{eq:cosine} follows from the definition of $q_k$. The result in Eq.\ \ref{eq:argmin} is the mean value of the cosine similarity $u_k$.
	
	Due to the symmetry of the cosine similarity, and as we prove formally at the end of this Theorem's Proof, it follows that
	\begin{gather}
		\prescript{}{opt}{}\boldsymbol{w}_k =\arg \max_{ \boldsymbol{w}_k } \mu_{p_k}\left(\cos\left( \boldsymbol{w}_k, \boldsymbol{x}\right)\right)=\arg \max_{ \boldsymbol{w}_k } \cos\left( \boldsymbol{w}_k, \mu_{p_k}\left(\boldsymbol{x}\right)\right)\label{eq:symmetry}
		\\=c\cdot \mu_{p_k}\left(\boldsymbol{x}\right), c\in\mathbb{R}^+.
		\label{eq:optweights}
	\end{gather}

	Enforcement of the requirement for normalization of the vector leads to the unique solution \\$ \prescript{}{opt}{}\boldsymbol{w}^*_k=\frac{\mu_{p_k}\left(\boldsymbol{x}\right)}{||\mu_{p_k}\left(\boldsymbol{x}\right)||}$, which proves the Theorem.

\textbf{Regarding Eq.\ \ref{eq:symmetry}:} It may not be obvious how Eq.\ \ref{eq:symmetry} follows from the symmetry of the cosine similarity, therefore in the remaining parts of this Theorem's Proof we prove it formally.

We define $\boldsymbol{w}_k^{(0)}$ as
\begin{equation}
\boldsymbol{w}_k^{(0)}\coloneqq \arg \max_{ \boldsymbol{w}_k } \cos\left( \boldsymbol{w}_k, \mu_{p_k}\left(\boldsymbol{x}\right)\right)=c\mu_{p_k}\left(\boldsymbol{x}\right), c\in\mathbb{R}^+. \label{eq:def_wk0}
\end{equation}

Equivalently to Eq.\ \ref{eq:symmetry}, we must prove that $\arg \max_{ \boldsymbol{w}_k } \mu_{p_k}\left(\cos\left( \boldsymbol{w}_k, \boldsymbol{x}\right)\right)=\boldsymbol{w}_k^{(0)}$, i.e.\ that $\max_{ \boldsymbol{w}_k } \mu_{p_k}\left(\cos\left( \boldsymbol{w}_k, \boldsymbol{x}\right)\right)=\mu_{p_k}\left(\cos\left( \boldsymbol{w}_k^{(0)}, \boldsymbol{x}\right)\right)$.

For this, we will show that, if the cosine is measured with respect to a weight vector shifted by some $\delta \boldsymbol{x}$ instead of with respect to $\boldsymbol{w}_k^{(0)}$, then the mean value decreases, i.e.\ that
\begin{gather}
\mu_{p_k}\left(\cos\left( \boldsymbol{w}_k^{(0)}, \boldsymbol{x}\right)\right)\geq\mu_{p_k}\left(\cos\left( \boldsymbol{w}_k^{(0)}+\delta \boldsymbol{x}, \boldsymbol{x}\right)\right),\quad \forall \delta \boldsymbol{x} \\
\iff \mu_{p_k}\left(\cos\left( \boldsymbol{w}_k^{(0)}, \boldsymbol{x}\right)\right)-\mu_{p_k}\left(\cos\left( \boldsymbol{w}_k^{(0)}+\delta \boldsymbol{x}, \boldsymbol{x}\right)\right)\geq 0,\quad \forall \delta \boldsymbol{x}. \label{eq:deltamean}
\end{gather}

Let 
$\cos_0\left(\boldsymbol{x}\right)\coloneqq \cos\left(\boldsymbol{w}_k^{(0)}, \boldsymbol{x}\right)$ and the shifted cosine function $\cos_1\left(\boldsymbol{x}\right) \coloneqq \cos\left(\boldsymbol{w}_k^{(0)}+\delta \boldsymbol{x}, \boldsymbol{x}\right)$.

Then, equivalently to inequality \ref{eq:deltamean}, we must show that
\begin{gather}
\mu_{p_k}\left(\cos_0\left(\boldsymbol{x}\right)\right)-\mu_{p_k}\left(\cos_1\left(\boldsymbol{x}\right)\right) \geq 0\\
\iff\int_{-\infty}^{\infty} \cos_0\left(\boldsymbol{x}\right) p_k \left(\boldsymbol{x}\right) d\boldsymbol{x} - \int_{-\infty}^{\infty} \cos_1\left(\boldsymbol{x}\right) p_k \left(\boldsymbol{x}\right) d\boldsymbol{x} \geq 0\\
\iff \int_{-\infty}^{\infty} \Delta\text{cos}\left(\boldsymbol{x}\right) p_k \left(\boldsymbol{x}\right) d\boldsymbol{x} \geq 0\\
\iff\int_{-\infty}^{{\mu_{p_k}\left(\boldsymbol{x}\right)+\frac{\delta \boldsymbol{x}}{2}}} \left(\Delta\text{cos}\circ p_k\right) \left(\boldsymbol{x}\right) d\boldsymbol{x}
+ \int_{{\mu_{p_k}\left(\boldsymbol{x}\right)+\frac{\delta \boldsymbol{x}}{2}}}^{\infty} \left(\Delta\text{cos}\circ p_k\right) \left(\boldsymbol{x}\right) d\boldsymbol{x} \geq 0\\
\iff\int_{-\infty}^{\boldsymbol{\nu}} \left(\Delta\text{cos}\circ p_k\right) \left(\boldsymbol{x}\right) d\boldsymbol{x}
- \int_{\boldsymbol{\nu}}^{\infty} \left(-\Delta\text{cos}\circ p_k\right) \left(\boldsymbol{x}\right) d\boldsymbol{x} \geq 0\\
\iff\int_{-\infty}^{\boldsymbol{\nu}} \left(\Delta\text{cos}\circ p_k\right) \left(\boldsymbol{x}\right) d\boldsymbol{x}
\geq \int_{\boldsymbol{\nu}}^{\infty} \left(-\Delta\text{cos}\circ p_k\right) \left(\boldsymbol{x}\right) d\boldsymbol{x}, \label{eq:ineq}
\end{gather}
where we have defined $\Delta\text{cos} \coloneqq \left(\cos_0-\cos_1\right)\left(\boldsymbol{x}\right)$, and $\boldsymbol{\nu}\coloneqq {\mu_{p_k}\left(\boldsymbol{x}\right)+\frac{\delta \boldsymbol{x}}{2}}$.

To show inequality \ref{eq:ineq}, it suffices to show that
\begin{equation}
	\left\{
	\begin{aligned}
p_k\left(\boldsymbol{\nu-\eta}\right) &\geq p_k\left(\boldsymbol{\nu+\eta}\right)
\\	\text{and } \Delta\text{cos}\left(\boldsymbol{\nu-\eta}\right)&=-\Delta\text{cos} \left(\boldsymbol{\nu+\eta} \right),
\end{aligned}\right.
	\end{equation}
$\forall \boldsymbol{\eta}: \delta \boldsymbol{x}\cdot \boldsymbol{\eta}\geq0$.

The probability density function $p_k \left(\boldsymbol{x}\right)$ is symmetrically decreasing, i.e.\ it is decreasing in the directions that point away from its mean $\mu_{p_k}\left(\boldsymbol{x}\right)$, with respect to any reference point. Therefore, taking as a reference the point $\boldsymbol{\nu}$, it is indeed true that
\begin{equation}p_k \left(\boldsymbol{\nu-\eta}\right) \geq p_k \left(\boldsymbol{\nu+\eta} \right), \quad \forall \boldsymbol{\eta}: \delta \boldsymbol{x}\cdot \boldsymbol{\eta}\geq0.
\end{equation}

We will now complete the proof by also showing the second equality that we seek, i.e.\ $\Delta\text{cos}\left(\boldsymbol{\nu-\eta}\right)=-\Delta\text{cos} \left(\boldsymbol{\nu+\eta} \right)$.

\begin{gather}
	 \Delta\text{cos}\left(\boldsymbol{\nu-\eta}\right) \nonumber
	 \\=\cos_0 \left(\boldsymbol{\nu-\eta}\right) - \cos_1 \left(\boldsymbol{\nu-\eta}\right)
     \\=\cos_0 \left(\mu_{p_k}\left(\boldsymbol{x}\right)+\frac{\delta \boldsymbol{x}}{2}-\boldsymbol{\eta}\right) - \cos_1 \left(\mu_{p_k}\left(\boldsymbol{x}\right)+\frac{\delta \boldsymbol{x}}{2}-\boldsymbol{\eta}\right).
\end{gather}

From the definition of $\boldsymbol{w}_k^{(0)}$, it is $\boldsymbol{w}_k^{(0)}=c\mu_{p_k}\left(\boldsymbol{x}\right), c\in\mathbb{R}^+$.
In addition, based on the definitions of $\cos_0$ and $\cos_1$, it is $\cos_0 \left(\boldsymbol{w}_k^{(0)}+\boldsymbol{\epsilon}\right)=\cos_1 \left(\boldsymbol{w}_k^{(0)}+\delta \boldsymbol{x}+\boldsymbol{\epsilon}\right), \forall \boldsymbol{\epsilon}$. Moreover, $\cos_0$ and $cos_1$ are symmetric around $\boldsymbol{w}_k^{(0)}$ and $\boldsymbol{w}_k^{(0)}+\delta \boldsymbol{x}$ respectively by their definition.
Therefore, by using these facts and by also choosing $c=1$, it is

\begin{gather}
	\Delta\text{cos}\left(\boldsymbol{\nu-\eta}\right) \nonumber
	\\=\cos_0 \left(\boldsymbol{w}_k^{(0)}+\frac{\delta \boldsymbol{x}}{2}-\boldsymbol{\eta}\right) - \cos_1 \left(\boldsymbol{w}_k^{(0)}+\frac{\delta \boldsymbol{x}}{2}-\boldsymbol{\eta}\right)
	\\=\cos_1 \left(\boldsymbol{w}_k^{(0)}+\delta \boldsymbol{x}+\frac{\delta \boldsymbol{x}}{2}-\boldsymbol{\eta}\right) - \cos_1 \left(\boldsymbol{w}_k^{(0)}+\frac{\delta \boldsymbol{x}}{2}-\boldsymbol{\eta}\right)
	\\=\cos_1 \left(\boldsymbol{w}_k^{(0)}+\delta \boldsymbol{x}-\left(\frac{\delta \boldsymbol{x}}{2}-\boldsymbol{\eta}\right)\right) - \cos_1 \left(\boldsymbol{w}_k^{(0)}+\frac{\delta \boldsymbol{x}}{2}-\boldsymbol{\eta}\right)
	\\=\cos_1 \left(\boldsymbol{w}_k^{(0)}+\frac{\delta \boldsymbol{x}}{2}+\boldsymbol{\eta}\right) - \cos_1 \left(\boldsymbol{w}_k^{(0)}+\delta \boldsymbol{x}-\frac{\delta \boldsymbol{x}}{2}-\boldsymbol{\eta}\right)
	\\=\cos_1 \left(\boldsymbol{w}_k^{(0)}+\frac{\delta \boldsymbol{x}}{2}+\boldsymbol{\eta}\right) - \cos_0 \left(\boldsymbol{w}_k^{(0)}-\frac{\delta \boldsymbol{x}}{2}-\boldsymbol{\eta}\right)
	\\=\cos_1 \left(\boldsymbol{w}_k^{(0)}+\frac{\delta \boldsymbol{x}}{2}+\boldsymbol{\eta}\right) - \cos_0 \left(\boldsymbol{w}_k^{(0)}+\frac{\delta \boldsymbol{x}}{2}+\boldsymbol{\eta}\right).
\end{gather}

Therefore, 
\begin{gather}
	\Delta\text{cos}\left(\boldsymbol{\nu-\eta}\right) \nonumber
	\\=-\Delta\text{cos}\left(\boldsymbol{w}_k^{(0)}+\frac{\delta \boldsymbol{x}}{2}+\boldsymbol{\eta}\right)
	\\=-\Delta\text{cos}\left(\boldsymbol{\nu+\eta}\right),
\end{gather}
which completes the Proof.
\end{proof}

\begin{proof}[\textbf{Proof of Theorem \ref{thm:equilib}}]
	We will find the equilibrium point of the SoftHebb plasticity rule, i.e.\ the weight $w_{ik}$ that implies $E[\Delta w_{ik}^{(SoftHebb)}]=0$.
	
	We will expand this expected value based on the plasticity rule itself, and on the probability distribution of the input $\boldsymbol{x}$.

	\begin{gather}
		E[\Delta w_{ik}^{(SoftHebb)}]
		=\eta\int_{\boldsymbol{x}} y_k(\boldsymbol{x})\cdot (x_i-u_k(\boldsymbol{x})w_{ik}) p(\boldsymbol{x})d\boldsymbol{x}
		\\=\eta\int_{\boldsymbol{x}} y_k(\boldsymbol{x}) (x_i-\boldsymbol{w}_k\boldsymbol{x}w_{ik}) \left(\sum_{l=1}^K p_l(\boldsymbol{x})P(C_l)\right) d\boldsymbol{x}
		\\=\eta \left[\rule{0cm}{0.7cm}\sum_{l=1}^K \int_{\boldsymbol{x}}x_i y_k(\boldsymbol{x}) p_l(\boldsymbol{x}) P(C_l)d\boldsymbol{x} -\sum_{l=1}^K\boldsymbol{w}_kw_{ik}\int_{\boldsymbol{x}} \boldsymbol{x}y_k(\boldsymbol{x})p_l(\boldsymbol{x})P(C_l) d\boldsymbol{x} \right]. \label{eq:equilibrium}
	\end{gather}
	
	Based on this, we will now show that 
	\begin{gather}
		\left[\boldsymbol{w}_k=\prescript{}{opt}{}\boldsymbol{w}^*_k=\frac{\mu_{p_k}(\boldsymbol{x})}{||\mu_{p_k}(\boldsymbol{x})||} \text{ and } w_{0k}=\prescript{}{opt}{}w_{0k},\, \forall k\right]\label{eq:premise}
		\\\implies E[\Delta w_{ik}^{(SoftHebb)}]=0  \quad \forall i,k.
	\end{gather}
	
	Using the premise \ref{eq:premise}, we can take the following steps, where steps 3, 4, and 6 are the main ones, and steps 1, 2, and 5, as well as Theorem \ref{thm:symmetricmeans} support those.
	
	\begin{enumerate}
		\item The cosine similarity function $u_k(\boldsymbol{x})=\boldsymbol{x}\boldsymbol{w}_k$, as determined by the total weighted input to the neuron $k$, and appropriately normalized, defines a probability distribution centred symmetrically around the vector $\boldsymbol{w}_k$, i.e.\ $\mu_{u_k}(\boldsymbol{x})=\boldsymbol{w}_k$ and $u_k(\mu_{u_k}-\boldsymbol{x})=u_k(\mu_{u_k}+\boldsymbol{x})$. \label{item:point1}
		
		$\boldsymbol{w}_k$, as premised, is equal to the normalized $\mu_{p_k}(\boldsymbol{x})$, i.e.\ the mean of the distribution $p_k(\boldsymbol{x})=p(\boldsymbol{x}|C_k)$, therefore: $\mu_{u_k}(\boldsymbol{x})=\mu_{p_k}(\boldsymbol{x})$.
		
		\item The soft-WTA version of the neuronal transformation, i.e.\ the softmax version of the model's inference is
		\begin{equation}y_k(\boldsymbol{x})=\frac{\exp\left(u_k(\boldsymbol{x})\right)\exp(w_{0k})}{\sum_{l=1}^K \exp\left(u_l(\boldsymbol{x})\right)\exp(w_{0l})},\label{eq:softmax}
		\end{equation}
		But because of the premise \ref{eq:premise} that the parameters of the model $u_k$ are set to their optimal value, it follows that $\exp(u_k(\boldsymbol{x}))=p_k(\boldsymbol{x})$ and $\exp(w_{0k})=P(C_k),\, \forall k$ (see also Theorem \ref{thm:nesslerbias}), therefore
		\begin{equation}y_k(\boldsymbol{x})=\frac{p_k(\boldsymbol{x})P(C_k)}{\sum_{l=1}^K p_l(\boldsymbol{x})P(C_l)}.\label{eq:softmax_p}
		\end{equation}

		\item Eq.\ \ref{eq:equilibrium} involves twice the function $y_k(\boldsymbol{x})p(\boldsymbol{x})=\sum_{l=1}^K y_k(\boldsymbol{x}) p_l(\boldsymbol{x})P(C_l)$. Using the above two points, we will now show that this is approximately equal to $y_k(\boldsymbol{x}) p_k(\boldsymbol{x})P(C_k)$.
		
		\begin{enumerate}
			\item We assume that the ``support'' $\mathbb{O}_k^p$ of the component $p_k(\boldsymbol{x})$, i.e.\ the region where $p_k$ is not negligible, is not fully overlapping with that of other components $p_l$.
			
			In addition, $\mathbb{O}_k^p$ is narrow relative to the input space $\forall k$, because, first, the cosine similarity $u_k(\boldsymbol{x})$ diminishes fast from its maximum at $\arg \max u_{k}=\boldsymbol{w}_k$ in case of non-trivial dimensionality of the input space, and, second, $p_k(\boldsymbol{x})=\exp\left(u_k(\boldsymbol{x})\right)$ applies a further exponential decrease. \label{item:supports}
			
			Therefore, the overlap $\mathbb{O}_k^p\cap\mathbb{O}_l^p$ is small, or none,  $\forall l\neq k$.
			
			If $\mathbb{O}_k^y$ is the ``support'' of $y_k$, then this is even narrower than $\mathbb{O}_k^p$, due to the softmax.
			
			As a result, the overlap $\mathbb{O}_k^y\cap\mathbb{O}_l^p$ of $y_k$ and $p_l$ is even smaller than the overlap $\mathbb{O}_k^p\cap\mathbb{O}_l^p$ of $p_k$ and $p_l$, $\forall l\neq k$.
			
			\item Because of the numerator in Eq.\ \ref{eq:softmax_p}, the overlap $\mathbb{O}_k^y\cap\mathbb{O}_k^p$ of $y_k$ and $p_k$ is large.
			
			Based on these two points, the overlaps $\mathbb{O}_k^y\cap\mathbb{O}_l^p$ can be neglected for $l\neq k$, and it follows that
			\begin{equation}\sum_{l=1}^K y_k(\boldsymbol{x}) p_l(\boldsymbol{x})P(C_l)\approx y_k(\boldsymbol{x}) p_k(\boldsymbol{x})P(C_k).
				\label{eq:approx}
			\end{equation}
		\end{enumerate}
		
		Therefore, we can write Eq.\ \ref{eq:equilibrium} as
		\begin{gather}
			E[\Delta w_{ik}^{(SoftHebb)}]\nonumber
			\\\approx \eta \left[\rule{0cm}{0.7cm}\int_{\boldsymbol{x}}x_i y_k(\boldsymbol{x}) p_k(\boldsymbol{x}) P(C_k)d\boldsymbol{x} -\boldsymbol{w}_kw_{ik}\int_{\boldsymbol{x}} \boldsymbol{x}y_k(\boldsymbol{x})p_k(\boldsymbol{x}) P(C_k)d\boldsymbol{x} \right]. \label{eq:equilibrium_approx}
		\end{gather}
		
		Next, we aim to show that the integrals $\int_{\boldsymbol{x}}x_i y_k(\boldsymbol{x}) p_k(\boldsymbol{x}) d\boldsymbol{x}$ and $\int_{\boldsymbol{x}}\boldsymbol{x} y_k(\boldsymbol{x}) p_k(\boldsymbol{x}) d\boldsymbol{x}$ involved in that Equation equal the mean value of $x_i$ and $\boldsymbol{x}$ respectively according to the distribution $p_k$.
		
		To show this, we observe that the integrals are indeed mean values of $x_i$ and $\boldsymbol{x}$ according to a probability distribution, and specifically the distribution $y_kp_k$. We will first show that the distribution $p_k$ is symmetric around its mean $\mu_{p_k}(\boldsymbol{x})$. Then we will show that $y_k(\boldsymbol{x})$ is also symmetric around the same mean. Then we will use the fact that the product of two such symmetric distributions with common mean, such as $y_kp_k$, is a distribution with the same mean, a fact that we will prove in Theorem A.1.

		\item 
		Because the cosine similarity function $u_k(\boldsymbol{x})$ is symmetric around the mean value $\mu_{u_k}(\boldsymbol{x})=\boldsymbol{w}_k$: $u_k(\mu_{u_k}-\boldsymbol{x})=u_k(\mu_{u_k}+\boldsymbol{x})$, it follows that  $p_k(\mu_{u_k}-\boldsymbol{x})=\exp(u_k(\mu_{u_k}-\boldsymbol{x}))=\exp(u_k(\mu_{u_k}+\boldsymbol{x}))=p_k(\mu_{u_k}+\boldsymbol{x})$.
		
		Therefore, $p_k(\boldsymbol{x})=\exp\left(u_k(\boldsymbol{x})\right)$ does have the sought property of symmetry, around $\mu_{p_k}(\boldsymbol{x})$.
		
		In point \ref{item:point1} of this list we have also shown that $\mu_{u_k}(\boldsymbol{x})=\mu_{p_k}(\boldsymbol{x})$, thus $p_k$ is symmetric around its own mean $\mu_{p_k}(\boldsymbol{x})$.
		
		\item The reason why the softmax output $y_k$ is symmetric around the same mean as $p_k$ consists in the following arguments:
		\begin{enumerate}
			\item The numerator $p_k$ of the $y_k$ softmax in Eq.\ \ref{eq:softmax_p} is symmetric around $\mu_{p_k}(\boldsymbol{x})$, as was shown in the preceding point.
			\item The denominator of Eq.\ \ref{eq:softmax_p}, i.e.\ $\sum_{l=1}^K p_l(\boldsymbol{x})P(C_l)$ is also symmetric around $\mu_{p_k}(\boldsymbol{x})$ in $\mathbb{O}_k^p$, where $\mathbb{O}_k^p$ is the ``support'' of the $p_k$ distribution, i.e.\ the region where the numerator $p_k$ is not negligible.
			
			This is because:
			\begin{enumerate}
				\item We assume that the data is distributed on the unit hypersphere of the input space according to $p(\boldsymbol{x})$ without a bias. Therefore the total contribution of components $\sum_{l\neq k}p_l(\boldsymbol{x})P(C_l)$ to $p(\boldsymbol{x})$ in that neighbourhood $\mathbb{O}_k^p$ is approximately symmetric around $\mu_{p_k}(\boldsymbol{x})$.
				
				\item $\sum_{l\neq k}p_l(\boldsymbol{x})P(C_l)$ is not only approximately symmetric, but also its remaining asymmetry has a negligible contribution to $p(\boldsymbol{x})$ in $\mathbb{O}_k^p$, because $p(x)$ in $\mathbb{O}_k^p$ is mostly determined by $p_k(\boldsymbol{x})$.
				
				This is true, because as we showed in point \ref{item:supports}, $\mathbb{O}_k^p \cap \mathbb{O}_l^p$ is a narrow overlap $\forall l\neq k$.
				
				\item $p_k$ is also symmetric, therefore, the total sum $\sum_{l=1}^K p_l(\boldsymbol{x})P(C_l)$ is symmetric around $\mu_{p_k}(\boldsymbol{x})$.
				
			\end{enumerate} 
			
			\item The inverse fraction $\frac{1}{f}$ of a symmetric function $f$ is also symmetric around the same mean, therefore the inverse of the denominator $\frac{1}{p(\boldsymbol{x})}=\frac{1}{\sum_{l=1}^K p_l(\boldsymbol{x})P(C_l)}$ is also symmetric around the mean value $\mu_{p_k}(\boldsymbol{x})$.
			
			\item The normalized product of two distributions that are symmetric around the same mean is a probability distribution with the same mean. We prove this formally in Theorem \ref{thm:symmetricmeans} and in its Proof at the end of the present Appendix \ref{sec:proofs}.
			
			Therefore $y_k=p_k \frac{1}{p(\boldsymbol{x})}$ is indeed symmetric around $\mu_{y_k}(\boldsymbol{x})=\mu_{u_k}(\boldsymbol{x})=\boldsymbol{w}_k=\mu_{p_k}(\boldsymbol{x})$.
		\end{enumerate}
		
		In summary, \begin{equation}\mu_{p_k}(\boldsymbol{x})=\mu_{y_k}(\boldsymbol{x}),
		\end{equation}
		and, due to the symmetry, \begin{equation}p_k(\mu_{p_k}+\boldsymbol{x})=p_k(\mu_{p_k}-\boldsymbol{x}),\end{equation} \begin{equation}y_k(\mu_{p_k}+\boldsymbol{x})=y_k(\mu_{p_k}-\boldsymbol{x}).\end{equation}
		
		\item Because the means of $y_k(\boldsymbol{x})$ and of $p_k(\boldsymbol{x})$ are both equal to $\mu_{p_k}(\boldsymbol{x})$, and because both distributions are symmetric around that mean, the probability distribution $y_k(\boldsymbol{x})p_k(\boldsymbol{x})P(C_k)/I_k$, where $I_k$ is the normalization constant, also has a mean $\mu_{y_kp_k}(\boldsymbol{x})$ equal to $\mu_{p_k}(\boldsymbol{x})$:
		
		\begin{equation}\int_{\boldsymbol{x}} \boldsymbol{x} y_k(\boldsymbol{x})p_k(\boldsymbol{x})P(C_k)/I_k\, d\boldsymbol{x}= \mu_{p_k}(\boldsymbol{x}).
		\end{equation}
		We prove this formally in Theorem \ref{thm:symmetricmeans} and in its Proof at the end of the present Appendix \ref{sec:proofs}.
		
		Therefore, the \textbf{first component} of the sum in Eq.\ \ref{eq:equilibrium_approx} is
		\begin{equation}\int_{\boldsymbol{x}} x_i y_k(\boldsymbol{x})p_k(\boldsymbol{x})P(C_k) d\boldsymbol{x}=I_k \cdot \mu_{p_k}(x_i)
		\end{equation}
		and, similarly, the \textbf{second component} is
		\begin{equation} -\boldsymbol{w}_kw_{ik}\int_{\boldsymbol{x}} \boldsymbol{x} y_k(\boldsymbol{x})p_k(\boldsymbol{x})P(C_k) d\boldsymbol{x}=-I_k \boldsymbol{w}_kw_{ik}\cdot \mu_{p_k}(\boldsymbol{x}).
		\end{equation}
	\end{enumerate}
	
	From the above conclusions about the two components of the sum in Eq.\ \ref{eq:equilibrium_approx}, it follows that 
	
	\begin{gather}
		E[\Delta w_{ik}^{(SoftHebb)}]=I_k \cdot \mu_{p_k}(x_i)-I_k \boldsymbol{w}_kw_{ik}\cdot \mu_{p_k}(\boldsymbol{x})
		\\=I_k \cdot (\mu_{p_k}(x_i) - \boldsymbol{w}_k \mu_{p_k}(\boldsymbol{x}) \cdot w_{ik})
		\\=I_k \cdot (\mu_{p_k}(x_i) - ||\mu_{p_k}(\boldsymbol{x})||\cdot w_{ik})
		\\=I_k \cdot \left(\mu_{p_k}(x_i) -||\mu_{p_k}(\boldsymbol{x})|| \frac{\mu_{p_k}(x_i)}{||\mu_{p_k}(\boldsymbol{x})||}\right)
		\\=0.
	\end{gather}
	
	Therefore, it is indeed true that $\left[\boldsymbol{w}_k=\prescript{}{opt}{}\boldsymbol{w}^*_k=\frac{\mu_{p_k}(\boldsymbol{x})}{||\mu_{p_k}(\boldsymbol{x})||} \, \forall k\right] \implies E[\Delta w_{ik}^{(SoftHebb)}]=0  \, \forall i,k$.
	
	Thus, the optimal weights of the model $\prescript{}{opt}{}\boldsymbol{w}^*_k=\frac{\mu_{p_k}(\boldsymbol{x})}{||\mu_{p_k}(\boldsymbol{x})||} \, \forall k$ are equilibrium weights of the SoftHebb plasticity rule and network.
	
	However, it is not yet clear that the weights that are normalized to a unit vector are those that the rule converges to, and that other norms of the vector are unstable. We will now give an intuition, and then prove that this is the case.
	
	The multiplicative factor $u_k$ is common between our rule and Oja's rule \citep{oja1982simplified}. The effect of this factor is known to normalize the weight vector of each neuron to a length of one \citep{oja1982simplified}, as also shown in similar rules with this multiplicative factor \citep{krotov2019unsupervised}. We prove that this is the effect of the factor also in the SoftHebb rule, separately in Theorem \ref{thm:SoftHebbnorm} and its Proof, provided at the end of the present Appendix \ref{sec:proofs}.
	
	This proves Theorem \ref{thm:equilib}, and satisfies the optimality condition derived in Theorem \ref{thm:optimal}.
\end{proof}

\begin{proof}[\textbf{Proof of Theorem \ref{thm:SoftHebbnorm}}] Using a technique similar to \citet{krotov2019unsupervised}, we write the SoftHebb plasticity rule as a differential equation \begin{equation}
		\tau \frac{dw_{ik}^{(SoftHebb)}}{dt}=\Delta w_{ik}^{(SoftHebb)}(t)=
		\eta \cdot y_k(t)  \cdot \left(x_i(t)-u_k(t)w_{ik}(t)\right). \label{eq:softhebb_differential}
	\end{equation}
	In this formulation, synapses undergo continuous changes, with an instantaneous rate $\Delta w_{ik}^{(SoftHebb)}(t)/\tau$. The time constant of the plasticity dynamics is $\tau$.
	
	The derivative of the norm of the weight vector is \begin{equation}
		\frac{d||\boldsymbol{w}_k||}{dt}=\frac{d(\boldsymbol{w}_k \boldsymbol{w}_k)}{dt}=2\boldsymbol{w}_k\frac{d\boldsymbol{w}_k }{dt}.
	\end{equation}
	Replacing $\frac{d\boldsymbol{w}_k }{dt}$ in this equation with the SoftHebb rule of Eq.\ \ref{eq:softhebb_differential}, it is
	\begin{align}
		\begin{split}
			\frac{d||\boldsymbol{w}_k^{SoftHebb)}||}{dt}=2\frac{\eta}{\tau}  \boldsymbol{w}_k \cdot y_k  \cdot \left(\boldsymbol{x}-u_k\boldsymbol{w}_k\right)&=2\frac{\eta}{\tau}  \boldsymbol{w}_k  \cdot y_k   \cdot \left(\boldsymbol{x} -\boldsymbol{w}_k \boldsymbol{x} \boldsymbol{w}_k \right)\\&=2\frac{\eta}{\tau}  u_k  y_k   \cdot \left(1-||\boldsymbol{w}_k ||\right).
		\end{split}
	\end{align}
	This differential equation shows that the derivative of the norm of the weight vector increases if $||\boldsymbol{w}_k||<1$ and decreases if $||\boldsymbol{w}_k||>1$, such that the weight vector tends to a sphere of radius $1$, which proves the Theorem.	
\end{proof}

\begin{proof}[\textbf{Proof of Theorem \ref{thm:nesslerbias}}]
	Similarly to the Proof of Theorem \ref{thm:equilib}, we find the equilibrium parameter $w_{0k}$ of the SoftHebb plasticity rule.

	\begin{gather}
		E[\Delta w_{0k}^{(SoftHebb)}]
		\\=\eta \int_{\boldsymbol{x}} \left(y_ke^{-w_{0k}} -1\right) p(\boldsymbol{x}) d\boldsymbol{x}
		\\=\eta\left(e^{-w_{0k}}\int_{\boldsymbol{x}} y_k(\boldsymbol{x}) p(\boldsymbol{x}) d\boldsymbol{x} -1\right)
		\\=\eta\left(e^{-w_{0k}}\int_{\boldsymbol{x}} \frac{p(\boldsymbol{x}|C_k)P(C_k)}{p(\boldsymbol{x})} p(\boldsymbol{x}) d\boldsymbol{x} -1\right)
		\\=\eta\left(e^{-w_{0k}}P(C_k) \int_{\boldsymbol{x}} p(\boldsymbol{x}|C_k)d\boldsymbol{x} -1\right)
		\\=\eta\left(e^{-w_{0k}}P(C_k) -1\right),
	\end{gather}
	
	In the above, we have replaced $y_k$ by its definition, i.e.\ $y_k(\boldsymbol{x})=p(C_k|\boldsymbol{x})=\frac{p(\boldsymbol{x}|C_k)P(C_k)}{p(\boldsymbol{x})} p(\boldsymbol{x})$.
	
	Therefore, using this form of $E[\Delta w_{0k}^{(SoftHebb)}]$, and setting this expectation to zero as a condition for equilibrium, we find the equilibrium value of $w_{0k}$:
	\begin{gather}
		E[\Delta w_{0k}^{(SoftHebb)}]=0 \implies \nonumber
		\\w_{0k}^{SoftHebb}=\ln P(C_k),
	\end{gather}
	which proves Theorem \ref{thm:nesslerbias} and shows the SoftHebb plasticity rule of the neuronal bias finds the optimal parameter of the Bayesian generative model as defined by Eq.\ \ref{eq:G0} of Theorem \ref{thm:optimal}.
\end{proof}
\begin{thm}
	\label{thm:symmetricmeans}
	Given two probability density functions (PDF) $y(x)$ and $p(x)$ that are both centred symmetrically around the same mean value $\mu$, their product $y(x)p(x)$, normalized appropriately, is a PDF with the same mean, i.e.
	\begin{equation}
		\begin{cases}
			\int_{x}xy(x)dx=\mu \\
			\int_{x}xp(x)dx=\mu \\
			p(\mu+x)=p(\mu-x) \\
			y(\mu+x)=y(\mu-x)
		\end{cases}
		\\ \implies
		\frac{1}{I}\int_{x} x y(x)p(x)dx=\mu.
	\end{equation}
\end{thm}
\begin{proof}[\textbf{Proof of Theorem \ref{thm:symmetricmeans}}]
	\begin{gather}
		\int_{-\infty}^{+\infty} x y(x)p(x)dx
		=\int_{-\infty}^{\mu} x y(x)p(x)dx +\int_{\mu}^{+\infty} x y(x)p(x)dx
	\end{gather}
	We will derive a different form for each of these two integrals.
	\begin{gather}
		\begin{align}
		I_1 &\coloneqq\int_{x=-\infty}^{\mu} x y(x)p(x)dx 
		\\ &=\int_{x=-\infty}^{\mu} (\mu- u) y(\mu- u)p(\mu- u)d(\mu- u) \\ \tag*{We defined $x=\mu- u$.} 
		\\ &=\int_{ u=+\infty}^{0} (\mu- u) y(\mu- u)p(\mu- u)d(\mu- u)   \\ \tag*{We substituted the integration limits accordingly.}  
		\\ &=-\int_{+\infty}^{0} (\mu- u) y(\mu- u)p(\mu- u)du   \\ \tag*{Because $d(\mu- u)=-du$.}   
		\\ &=-\int_{+\infty}^{0} \mu y(\mu-u)p(\mu-u)du
		+\int_{+\infty}^{0} u y(\mu-u)p(\mu-u)du 
		\\ &=-\int_{u=-\infty}^{0} \mu y(\mu+u)p(\mu+u)d(-u)
		+\int_{u=+\infty}^{0} u y(\mu-u)p(\mu-u)du \\ \tag*{We substituted the first integration variable by $-u$ and changed the limits accordingly.} 
		\\ &=-\int_{-\infty}^{0} \mu y(\mu+u)p(\mu+u)d(-u)
		-\int_{0}^{+\infty} u y(\mu-u)p(\mu-u)du \\ \tag*{We inverted the direction of the second integration.} 
		\\ &=\int_{-\infty}^{0} \mu y(\mu+u)p(\mu+u)du
		-\int_{0}^{+\infty} u y(\mu-u)p(\mu-u)du  \\ \tag*{Because $d(-u)=-du$.} 
		\\ &=\int_{-\infty}^{0} \mu y(\mu-u)p(\mu-u)du
		-\int_{0}^{+\infty} u y(\mu+u)p(\mu+u)du  \\ \tag*{We used the symmetry of the two distributions around their mean.} 
		\\ &=\mu\int_{-\infty}^{\mu}  y(x)p(x)dx
		-\int_{0}^{+\infty} u y(\mu+u)p(\mu+u)du.   \\ \tag*{We substituted the variable and the limits by $x=\mu- u$.} 
	\end{align}
	\end{gather}
	
	\begin{gather}
		\begin{align}
		I_2 &\coloneqq\int_{\mu}^{+\infty} x y(x)p(x)dx 
		\\ &=\int_{0}^{+\infty} (\mu+u) y(\mu+u)p(\mu+u)d(\mu+u) 
		\\ &=\int_{0}^{+\infty} \mu y(\mu+u)p(\mu+u)du+\int_{0}^{+\infty} u y(\mu+u)p(\mu+u)du 
		\\ &=\mu\int_{\mu}^{+\infty}  y(x)p(x)dx+\int_{0}^{+\infty} u y(\mu+u)p(\mu+u)du. 
	\end{align}
	\end{gather}
	
	Therefore,
	\begin{gather}
		\int_{-\infty}^{+\infty} x y(x)p(x)dx
		=I_1+I_2
		\\=\mu\int_{-\infty}^{\mu}  y(x)p(x)dx
		-\int_{0}^{+\infty} u y(\mu+u)p(\mu+u)du \\+\mu\int_{\mu}^{+\infty}  y(x)p(x)dx+\int_{0}^{+\infty} u y(\mu+u)p(\mu+u)du= \mu\cdot I,
	\end{gather}
	where $I=\int_{-\infty}^{+\infty} y(x)p(x)dx$.
	
\end{proof}
\subsection{Details to theoretical support of alternate activation functions (Section \ref{sec:activation_fn})}
\label{app:alternate_act}
Theorem \ref{thm:optimal}, which concerns the synaptic plasticity rule in Eq.\ \ref{eq:synplast}, was proven for the model of Definition \ref{def:model}, which uses a mixture of natural exponential component distributions, i.e.\ with base e (Eq.\ \ref{eq:multinomial}): \begin{equation}q_k\coloneqq q(\boldsymbol{x}|C_k; \boldsymbol{w}_k)=e^{u_k}.
\end{equation}
This implied an equivalence to a WTA neural network with natural exponential activation functions (Section \ref{sec:neuro_exp}).
However, it is simple to show that these results can be extended to other model probability distributions, and thus other neuronal activations.

Firstly, in the simplest of the alternatives, the base of the exponential function can be chosen differently. In that case, the posterior probabilities that are produced by the model's Bayesian inference, i.e.\ the network outputs, $Q(C_k|\boldsymbol{x};\boldsymbol{w})=y_k(\boldsymbol{x};\boldsymbol{w})$ are given by a softmax with a different base. If the base of the exponential is $b$, then \begin{equation}Q(C_k|\boldsymbol{x};\boldsymbol{w})=y_k=\frac{b^{u_k+w_{0k}}}{\sum_{l=1}^{K}b^{u_l+w_{0l}}}.
\end{equation}
It is obvious in the Proof of Theorem \ref{thm:optimal} in Appendix \ref{sec:proofs} that the same proof also applies to the changed base, if we use the appropriate logarithm for describing KL divergence. Therefore, the optimal parameter vector does not change, and the SoftHebb plasticity rule also applies to the SoftHebb model with a different exponential base. This change of the base in the softmax bears similarities to the change of its exponent, in a technique that is called Temperature Scaling and has been proven useful in classification \citep{hinton2015distilling}.

Secondly, the more conventional type of Temperature Scaling, i.e.\ that which scales exponent, is also possible in our model, while maintaining a Bayesian probabilistic interpretation of the outputs, a neural interpretation of the model, and the optimality of the plasticity rule. In this case, the model becomes \begin{equation}Q(C_k|\boldsymbol{x};\boldsymbol{w})=y_k=\frac{e^{(u_k+w_{0k})/T}}{\sum_{l=1}^{K}e^{(u_l+w_{0l})/T}}.
\end{equation}
The Proof of Theorem \ref{thm:optimal} in Appendix \ref{sec:proofs} also applies in this case, with a change in Eq.\ \ref{eq:cosine}, but resulting in the same solution. Therefore, the SoftHebb synaptic plasticity rule is applicable in this case too. The solution for the neuronal biases, i.e.\ the parameters of the prior in the Theorem (Eq.\ \ref{eq:G0}), also remains the same, but with a factor of $T$: 	$\prescript{}{opt}{}w_{0k}=T\ln P(C_k)$.

Finally, and most generally, the model can be generalized to use any non-negative and monotonically increasing function $h(x)$ for the component distributions, i.e.\ for the activation function of the neurons, assuming $h(x)$ is appropriately normalized to be interpretable as a probability density function. In this case the model becomes \begin{equation}Q(C_k|\boldsymbol{x};\boldsymbol{w})=y_k=\frac{h(u_k)\cdot w_{0k}}{\sum_{l=1}^{K}h(u_l)\cdot w_{0l}}.
\end{equation}
Note that there is a change in the parametrization of the priors into a multiplicative bias $\boldsymbol{w}_0$, compared to the additive bias in the previous versions above. This change is necessary in this general case, because not all functions have the property $e^{a+b}=e^a\cdot e^b$ that we used in the exponential case.
We can show that the optimal weight parameters remain the same as in the previous case of an exponential activation function, also for this more general case of activation $h$.
It can be seen in the Proof of Theorem \ref{thm:optimal}, that for a more general function $h(x)$ than the exponential, Eq.\ \ref{eq:cosine} would instead become:
\begin{gather}
	\prescript{}{opt}{}\boldsymbol{w}_k  =\arg \min_{ \boldsymbol{w}_k } D_{KL}(p_k||q_k) \nonumber
	=\arg \max_{ \boldsymbol{w}_k }\int_{\boldsymbol{x}}  p_k  \ln h(u_k)  d\boldsymbol{x}  \nonumber \\
	=\arg \max_{ \boldsymbol{w}_k }\int_{\boldsymbol{x}}  p_k  \ln h(\cos\left( \boldsymbol{w}_k, \boldsymbol{x}\right))  d\boldsymbol{x}	\nonumber \\
	=\arg \max_{ \boldsymbol{w}_k }\mu_{p_k}  \left(\ln h(\cos\left( \boldsymbol{w}_k, \boldsymbol{x}\right))\right) \nonumber\\
	=\arg \max_{ \boldsymbol{w}_k }\mu_{p_k}  \left(g(\cos\left( \boldsymbol{w}_k, \boldsymbol{x}\right))\right),	\label{eq:argmin_h_mu}
\end{gather}
where $g(x)=\ln h(x)$.
We have assumed that $h$ is an increasing function, therefore $g$ is also increasing.
The cosine similarity is symmetrically decreasing as a function of $\boldsymbol{x}$ around $\boldsymbol{w}_k$. Therefore, the function $g'(\boldsymbol{x})=g(\cos(\boldsymbol{w}_k, \boldsymbol{x}))$ also decreases symmetrically around $\boldsymbol{w}_k$. Thus, the mean of that function $g'$ under the probability distribution $p_k$ is maximum when $\mu_{p_k}=\boldsymbol{w}_k$.
As a result, Eq.\ \ref{eq:argmin_h_mu} implies that in this more general model too, the optimal weight vector is $\prescript{}{opt}{}\boldsymbol{w}_k=c\cdot \mu_{p_k}\left(\boldsymbol{x}\right), c\in\mathbb{R}$, and, consequently, it is also optimized by the same SoftHebb plasticity rule.

The implication of this is that the SoftHebb WTA neural network can use activation functions such as rectified linear units (ReLU), or other non-negative and increasing activations, such as rectified polynomials \citep{krotov2019unsupervised} etc., and maintain its generative properties, its Bayesian computation, and the theoretical optimality of the plasticity rule. A more complex derivation of the optimal weight vector for alternative activation functions, which was specific to ReLU only, and did not also derive the associated long-term plasticity rule for our problem category (Definition \ref{def:data}), was provided by \citet{moraitis2020shortterm}.

\subsection{Hyperparameters}
\label{sec:hyperparams}
We tuned the hyperparameters of the learning algorithms for each experiment. The tuned values that we found for MNIST are shown in the table below.

\begin{table}[h]
	\centering
	\begin{tabular}{rccc}
		& \textbf{Hard WTA} & \textbf{Soft WTA} & \textbf{Backpropagation} \\ \hline
		\multicolumn{4}{c}{\textbf{100 epochs}} \\ \hline
		Optimizer & - & - & Adam \\
		Softmax base & - & 1000 & - \\
		Initial learning rate & 0.05 & 0.03 & 0.001 \\
		Learning rate decay & Linear & Linear & - \\
		Minibatch size & 128 & 128 & 64 \\ \hline \hline
		\multicolumn{4}{c}{\textbf{1 epoch}} \\ \hline
		Optimizer & - & - & SGD \\
		Softmax base & - & 200 & - \\
		Initial learning rate & 0.55 & 0.55 & 0.2 \\
		Learning rate decay & Exponential & Exponential & - \\
		Minimum learning rate & 0.0055 & 0.0055 & - \\
		Minibatch size & 1 & 1 & 4 \\ \hline
	\end{tabular}
\caption{Training hyperparameters.}
\end{table}

For Fashion-MNIST, the same values are best, except that an initial learning rate value of 0.065 performed better for the WTA networks.

\subsection{Details to adversarial attack experiments}
\label{app:adversarial}
\begin{figure}[h]
	\centering
	\includegraphics[width = 140mm]{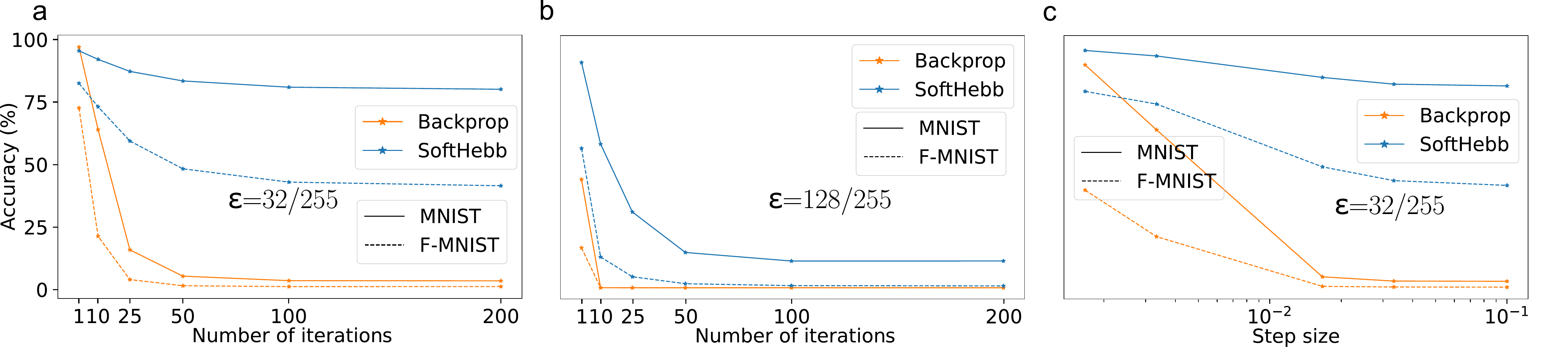}
	\caption{Adversarial-attack hyperparameter tuning. Accuracy as a function of number of PGD iterations, and of step size for different attack strength size $\epsilon$.}
	\label{fig:iterationcurves}
\end{figure}

We used ``Foolbox'', a Python library for adversarial attacks \citep{rauber2017foolbox}. PGD has a few parameters that influence the effectiveness of the attack. Namely, $\epsilon$ which is a parameter determining the size $\epsilon$ of the perturbation, the number of iterations for the attack's gradient ascent, the step size per iteration, and a number of possible random restarts per attacked sample. Here we chose 5 random restarts. Then we found that 200 iterations are sufficient for both MNIST and F-MNIST (Fig. \ref{fig:iterationcurves}a and b). Then, using 200 iterations, and different $\epsilon$ values, we searched for a sufficiently good step size. We found that relative to $\epsilon$ a step size value of $\frac{0.01}{0.3}\epsilon\approx0.33\epsilon$ (which is also the default value of the toolbox that we used) is a good value. An example curve for $\epsilon=32/255$ is shown in Fig. \ref{fig:iterationcurves}c for MNIST and F-MNIST.

\section{Discussion}
We have described SoftHebb, a highly biologically plausible neural algorithm that is founded on a Bayesian ML-theoretic framework. The model consists of elements fully compatible with conventional ANNs. It was previously not known which plasticity rule should be used to learn a Bayesian generative model of the input distribution in ANN WTA networks. Here we derived this rigorously. Moreover, we showed that Hard-WTA networks, and neurons with other activation functions can be described within the same framework as variations of the probabilistic model. This theory could provide a new foundation for normative Hebbian ANN designs with practical significance.

SoftHebb's properties are highly sought-after by efficient neuromorphic learning chips. It is unsupervised, local, and requires no error or other feedback currents from upper layers, thus solving hardware-inefficiencies and bio-implausibilities of backpropagation such as weight-transport and update-locking. However, a limitation of our study is that we have not experimented with deeper networks that could enable true applicability in complex datasets. This will require further experiments that are beyond the single-layer theoretical scope of the present work. A convolutional implementation could become the foundation of such experiments and could provide insights into the role of WTA microcircuits in larger cortical networks with localized receptive fields \citep{pogodin2021towards}, similar to area V1 of cortex \citep{hubel1962receptive}. This would be a radically different approach from backpropagation and its approximations, circumventing their key limitations by not relying on any feedback signals. The possibility that this may reach competitive accuracies in multilayer datasets becomes more realistic as a consequence of the present work. The ability of SoftHebb to minimize cross-entropy without supervision, and its fast and robust learning, are properties that may conceivably support a multilayer learning algorithm without any feedback signals \citep{journe2022hebbian}.

In addition to its future potential, surprisingly, SoftHebb already has practical advantages beyond its efficiency. Specifically, it surpasses backpropagation in accuracy, when training time and network size are limited. Moreover, in a demonstration that goes beyond the common greedy-training approach to Hebbian networks, SoftHebb achieves update-unlocked operation in practice, by updating the first layer before the input's processing by the next layer. It is intriguing that, through its biological plausibility, emerge properties commonly associated with biological intelligence, such as speed of learning, and substantially increased robustness to noise and adversarial attacks. Importantly, robustness emerges without specialized defences. Furthermore, SoftHebb tends to not merely be robust to attacks, but actually deflect them as specialized SOTA defences aim to do \citep{qin2020deflecting}. We also demonstrated the ability of SoftHebb to generate synthetic objects as interpolations of true object classes.

All in all, the algorithm has several properties that are individually interesting, novel, and worth future exploration.
Combined, however, SoftHebb's properties already enable certain applications that are small-scale but infeasible with backpropagation-based learning. For example, fast, on-line, unsupervised learning of simple tasks by edge-sensing neuromorphic devices, operating in noisy conditions, with a small battery and only local processing, requires those algorithmic properties that we have demonstrated here.

\bibliography{references.bib}

\begin{thebibliography}{73}
\providecommand{\natexlab}[1]{#1}
\providecommand{\url}[1]{\texttt{#1}}
\expandafter\ifx\csname urlstyle\endcsname\relax
  \providecommand{\doi}[1]{doi: #1}\else
  \providecommand{\doi}{doi: \begingroup \urlstyle{rm}\Url}\fi

\bibitem[Amato et~al.(2019)Amato, Carrara, Falchi, Gennaro, and
  Lagani]{amato2019hebbian}
Amato, G., Carrara, F., Falchi, F., Gennaro, C., and Lagani, G.
\newblock Hebbian learning meets deep convolutional neural networks.
\newblock In \emph{International Conference on Image Analysis and Processing},
  pp.\  324--334. Springer, 2019.

\bibitem[Bardes et~al.(2021)Bardes, Ponce, and LeCun]{bardes2021vicreg}
Bardes, A., Ponce, J., and LeCun, Y.
\newblock Vicreg: Variance-invariance-covariance regularization for
  self-supervised learning.
\newblock \emph{arXiv preprint arXiv:2105.04906}, 2021.

\bibitem[Bell \& Sejnowski(1995)Bell and Sejnowski]{bell1995information}
Bell, A.~J. and Sejnowski, T.~J.
\newblock An information-maximization approach to blind separation and blind
  deconvolution.
\newblock \emph{Neural computation}, 7\penalty0 (6):\penalty0 1129--1159, 1995.

\bibitem[Bengio et~al.(2015)Bengio, Lee, Bornschein, Mesnard, and
  Lin]{bengio2015towards}
Bengio, Y., Lee, D.-H., Bornschein, J., Mesnard, T., and Lin, Z.
\newblock Towards biologically plausible deep learning.
\newblock \emph{arXiv preprint arXiv:1502.04156}, 2015.

\bibitem[Berthelot et~al.(2018)Berthelot, Raffel, Roy, and
  Goodfellow]{berthelot2018understanding}
Berthelot, D., Raffel, C., Roy, A., and Goodfellow, I.
\newblock Understanding and improving interpolation in autoencoders via an
  adversarial regularizer.
\newblock \emph{arXiv preprint arXiv:1807.07543}, 2018.

\bibitem[Binas et~al.(2014)Binas, Rutishauser, Indiveri, and
  Pfeiffer]{binas2014learning}
Binas, J., Rutishauser, U., Indiveri, G., and Pfeiffer, M.
\newblock Learning and stabilization of winner-take-all dynamics through
  interacting excitatory and inhibitory plasticity.
\newblock \emph{Frontiers in computational neuroscience}, 8:\penalty0 68, 2014.

\bibitem[Binzegger et~al.(2004)Binzegger, Douglas, and
  Martin]{binzegger2004quantitative}
Binzegger, T., Douglas, R.~J., and Martin, K.~A.
\newblock A quantitative map of the circuit of cat primary visual cortex.
\newblock \emph{Journal of Neuroscience}, 24\penalty0 (39):\penalty0
  8441--8453, 2004.

\bibitem[Bittar \& Garner(2022)Bittar and Garner]{bittar2022surrogate}
Bittar, A. and Garner, P.~N.
\newblock A surrogate gradient spiking baseline for speech command recognition.
\newblock \emph{Frontiers in Neuroscience}, 16, 2022.

\bibitem[Bojanowski et~al.(2017)Bojanowski, Joulin, Lopez-Paz, and
  Szlam]{bojanowski2017optimizing}
Bojanowski, P., Joulin, A., Lopez-Paz, D., and Szlam, A.
\newblock Optimizing the latent space of generative networks.
\newblock \emph{arXiv preprint arXiv:1707.05776}, 2017.

\bibitem[Cannon et~al.(2014)Cannon, McCarthy, Lee, Lee, B{\"o}rgers,
  Whittington, and Kopell]{cannon2014neurosystems}
Cannon, J., McCarthy, M.~M., Lee, S., Lee, J., B{\"o}rgers, C., Whittington,
  M.~A., and Kopell, N.
\newblock Neurosystems: brain rhythms and cognitive processing.
\newblock \emph{European Journal of Neuroscience}, 39\penalty0 (5):\penalty0
  705--719, 2014.

\bibitem[Chen et~al.(2020)Chen, Kornblith, Norouzi, and Hinton]{chen2020simple}
Chen, T., Kornblith, S., Norouzi, M., and Hinton, G.
\newblock A simple framework for contrastive learning of visual
  representations.
\newblock In \emph{International conference on machine learning}, pp.\
  1597--1607. PMLR, 2020.

\bibitem[Cowen-Rivers et~al.(2022)Cowen-Rivers, Lyu, Tutunov, Wang, Grosnit,
  Griffiths, Maraval, Jianye, Wang, Peters, et~al.]{cowen2022hebo}
Cowen-Rivers, A.~I., Lyu, W., Tutunov, R., Wang, Z., Grosnit, A., Griffiths,
  R.~R., Maraval, A.~M., Jianye, H., Wang, J., Peters, J., et~al.
\newblock Hebo: Pushing the limits of sample-efficient hyper-parameter
  optimisation.
\newblock \emph{Journal of Artificial Intelligence Research}, 74:\penalty0
  1269--1349, 2022.

\bibitem[Creswell et~al.(2018)Creswell, White, Dumoulin, Arulkumaran, Sengupta,
  and Bharath]{creswell2018generative}
Creswell, A., White, T., Dumoulin, V., Arulkumaran, K., Sengupta, B., and
  Bharath, A.~A.
\newblock Generative adversarial networks: An overview.
\newblock \emph{IEEE Signal Processing Magazine}, 35\penalty0 (1):\penalty0
  53--65, 2018.

\bibitem[Crick(1989)]{crick1989recent}
Crick, F.
\newblock The recent excitement about neural networks.
\newblock \emph{Nature}, 337\penalty0 (6203):\penalty0 129--132, 1989.

\bibitem[Czarnecki et~al.(2017)Czarnecki, {\'S}wirszcz, Jaderberg, Osindero,
  Vinyals, and Kavukcuoglu]{czarnecki2017understanding}
Czarnecki, W.~M., {\'S}wirszcz, G., Jaderberg, M., Osindero, S., Vinyals, O.,
  and Kavukcuoglu, K.
\newblock Understanding synthetic gradients and decoupled neural interfaces.
\newblock In \emph{International Conference on Machine Learning}, pp.\
  904--912. PMLR, 2017.

\bibitem[Devlin et~al.(2018)Devlin, Chang, Lee, and Toutanova]{devlin2018bert}
Devlin, J., Chang, M.-W., Lee, K., and Toutanova, K.
\newblock Bert: Pre-training of deep bidirectional transformers for language
  understanding.
\newblock \emph{arXiv preprint arXiv:1810.04805}, 2018.

\bibitem[Diehl \& Cook(2015)Diehl and Cook]{diehl2015unsupervised}
Diehl, P.~U. and Cook, M.
\newblock Unsupervised learning of digit recognition using
  spike-timing-dependent plasticity.
\newblock \emph{Frontiers in computational neuroscience}, 9:\penalty0 99, 2015.

\bibitem[Diehl \& Cook(2016)Diehl and Cook]{diehl2016learning}
Diehl, P.~U. and Cook, M.
\newblock Learning and inferring relations in cortical networks.
\newblock \emph{arXiv preprint arXiv:1608.08267}, 2016.

\bibitem[Douglas \& Martin(2004)Douglas and Martin]{douglas2004neuronal}
Douglas, R.~J. and Martin, K.~A.
\newblock Neuronal circuits of the neocortex.
\newblock \emph{Annu. Rev. Neurosci.}, 27:\penalty0 419--451, 2004.

\bibitem[Ernoult et~al.(2020)Ernoult, Grollier, Querlioz, Bengio, and
  Scellier]{ernoult2020equilibrium}
Ernoult, M., Grollier, J., Querlioz, D., Bengio, Y., and Scellier, B.
\newblock Equilibrium propagation with continual weight updates.
\newblock \emph{arXiv preprint arXiv:2005.04168}, 2020.

\bibitem[F{\"o}ldiak(1990)]{foldiak1990forming}
F{\"o}ldiak, P.
\newblock Forming sparse representations by local anti-hebbian learning.
\newblock \emph{Biological cybernetics}, 64\penalty0 (2):\penalty0 165--170,
  1990.

\bibitem[F{\"o}ldi{\'a}k \& Fdilr(1989)F{\"o}ldi{\'a}k and
  Fdilr]{foldiak1989adaptive}
F{\"o}ldi{\'a}k, P. and Fdilr, P.
\newblock Adaptive network for optimal linear feature extraction.
\newblock 1989.

\bibitem[Frenkel et~al.(2021)Frenkel, Lefebvre, and Bol]{frenkel2021learning}
Frenkel, C., Lefebvre, M., and Bol, D.
\newblock Learning without feedback: Fixed random learning signals allow for
  feedforward training of deep neural networks.
\newblock \emph{Frontiers in neuroscience}, 15:\penalty0 629892, 2021.

\bibitem[Garcia~Rodriguez et~al.(2022)Garcia~Rodriguez, Guo, and
  Moraitis]{rodriguez2022short}
Garcia~Rodriguez, H., Guo, Q., and Moraitis, T.
\newblock Short-term plasticity neurons learning to learn and forget.
\newblock In \emph{International Conference on Machine Learning}, pp.\
  18704--18722. PMLR, 2022.

\bibitem[Goodfellow et~al.(2014)Goodfellow, Pouget-Abadie, Mirza, Xu,
  Warde-Farley, Ozair, Courville, and Bengio]{goodfellow2014generative}
Goodfellow, I.~J., Pouget-Abadie, J., Mirza, M., Xu, B., Warde-Farley, D.,
  Ozair, S., Courville, A., and Bengio, Y.
\newblock Generative adversarial networks.
\newblock \emph{arXiv preprint arXiv:1406.2661}, 2014.

\bibitem[Grinberg et~al.(2019)Grinberg, Hopfield, and
  Krotov]{grinberg2019local}
Grinberg, L., Hopfield, J., and Krotov, D.
\newblock Local unsupervised learning for image analysis.
\newblock \emph{arXiv preprint arXiv:1908.08993}, 2019.

\bibitem[Grossberg(1987)]{grossberg1987competitive}
Grossberg, S.
\newblock Competitive learning: From interactive activation to adaptive
  resonance.
\newblock \emph{Cognitive science}, 11\penalty0 (1):\penalty0 23--63, 1987.

\bibitem[Guerguiev et~al.(2017)Guerguiev, Lillicrap, and
  Richards]{guerguiev2017towards}
Guerguiev, J., Lillicrap, T.~P., and Richards, B.~A.
\newblock Towards deep learning with segregated dendrites.
\newblock \emph{ELife}, 6:\penalty0 e22901, 2017.

\bibitem[Hahnloser et~al.(1999)Hahnloser, Douglas, Mahowald, and
  Hepp]{hahnloser1999feedback}
Hahnloser, R., Douglas, R.~J., Mahowald, M., and Hepp, K.
\newblock Feedback interactions between neuronal pointers and maps for
  attentional processing.
\newblock \emph{nature neuroscience}, 2\penalty0 (8):\penalty0 746--752, 1999.

\bibitem[Hinton et~al.(2015)Hinton, Vinyals, and Dean]{hinton2015distilling}
Hinton, G., Vinyals, O., and Dean, J.
\newblock Distilling the knowledge in a neural network.
\newblock \emph{arXiv preprint arXiv:1503.02531}, 2015.

\bibitem[Hu et~al.(2014)Hu, Pehlevan, and Chklovskii]{hu2014hebbian}
Hu, T., Pehlevan, C., and Chklovskii, D.~B.
\newblock A hebbian/anti-hebbian network for online sparse dictionary learning
  derived from symmetric matrix factorization.
\newblock In \emph{2014 48th Asilomar Conference on Signals, Systems and
  Computers}, pp.\  613--619. IEEE, 2014.

\bibitem[Hubel \& Wiesel(1962)Hubel and Wiesel]{hubel1962receptive}
Hubel, D.~H. and Wiesel, T.~N.
\newblock Receptive fields, binocular interaction and functional architecture
  in the cat's visual cortex.
\newblock \emph{The Journal of physiology}, 160\penalty0 (1):\penalty0
  106--154, 1962.

\bibitem[Illing et~al.(2019)Illing, Gerstner, and Brea]{illing2019biologically}
Illing, B., Gerstner, W., and Brea, J.
\newblock Biologically plausible deep learning—but how far can we go with
  shallow networks?
\newblock \emph{Neural Networks}, 118:\penalty0 90--101, 2019.

\bibitem[Isomura \& Toyoizumi(2018)Isomura and Toyoizumi]{isomura2018error}
Isomura, T. and Toyoizumi, T.
\newblock Error-gated hebbian rule: A local learning rule for principal and
  independent component analysis.
\newblock \emph{Scientific reports}, 8\penalty0 (1):\penalty0 1--11, 2018.

\bibitem[Jeffares et~al.(2022)Jeffares, Guo, Stenetorp, and
  Moraitis]{jeffares2022spikeinspired}
Jeffares, A., Guo, Q., Stenetorp, P., and Moraitis, T.
\newblock Spike-inspired rank coding for fast and accurate recurrent neural
  networks.
\newblock In \emph{International Conference on Learning Representations}, 2022.
\newblock URL \url{https://openreview.net/forum?id=iMH1e5k7n3L}.

\bibitem[Journ{\'e} et~al.(2022)Journ{\'e}, Rodriguez, Guo, and
  Moraitis]{journe2022hebbian}
Journ{\'e}, A., Rodriguez, H.~G., Guo, Q., and Moraitis, T.
\newblock Hebbian deep learning without feedback.
\newblock \emph{arXiv preprint arXiv:2209.11883}, 2022.

\bibitem[Kingma \& Ba(2015)Kingma and Ba]{kingma2015adam}
Kingma, D.~P. and Ba, J.
\newblock Adam: {A} method for stochastic optimization.
\newblock In Bengio, Y. and LeCun, Y. (eds.), \emph{3rd International
  Conference on Learning Representations, {ICLR} 2015, San Diego, CA, USA, May
  7-9, 2015, Conference Track Proceedings}, 2015.
\newblock URL \url{http://arxiv.org/abs/1412.6980}.

\bibitem[Krizhevsky et~al.(2009)Krizhevsky, Hinton,
  et~al.]{krizhevsky2009learning}
Krizhevsky, A., Hinton, G., et~al.
\newblock Learning multiple layers of features from tiny images.
\newblock 2009.

\bibitem[Krotov \& Hopfield(2019)Krotov and Hopfield]{krotov2019unsupervised}
Krotov, D. and Hopfield, J.~J.
\newblock Unsupervised learning by competing hidden units.
\newblock \emph{Proceedings of the National Academy of Sciences}, 116\penalty0
  (16):\penalty0 7723--7731, 2019.

\bibitem[Lagani et~al.(2021)Lagani, Falchi, Gennaro, and
  Amato]{lagani2021hebbian}
Lagani, G., Falchi, F., Gennaro, C., and Amato, G.
\newblock Hebbian semi-supervised learning in a sample efficiency setting.
\newblock \emph{Neural Networks}, 143:\penalty0 719--731, 2021.

\bibitem[Lee et~al.(1999)Lee, Girolami, and Sejnowski]{lee1999independent}
Lee, T.-W., Girolami, M., and Sejnowski, T.~J.
\newblock Independent component analysis using an extended infomax algorithm
  for mixed subgaussian and supergaussian sources.
\newblock \emph{Neural computation}, 11\penalty0 (2):\penalty0 417--441, 1999.

\bibitem[Lillicrap et~al.(2016)Lillicrap, Cownden, Tweed, and
  Akerman]{lillicrap2016random}
Lillicrap, T.~P., Cownden, D., Tweed, D.~B., and Akerman, C.~J.
\newblock Random synaptic feedback weights support error backpropagation for
  deep learning.
\newblock \emph{Nature communications}, 7\penalty0 (1):\penalty0 1--10, 2016.

\bibitem[Linsker(1992)]{linsker1992local}
Linsker, R.
\newblock Local synaptic learning rules suffice to maximize mutual information
  in a linear network.
\newblock \emph{Neural Computation}, 4\penalty0 (5):\penalty0 691--702, 1992.

\bibitem[Maass(2000)]{maass2000computational}
Maass, W.
\newblock On the computational power of winner-take-all.
\newblock \emph{Neural computation}, 12\penalty0 (11):\penalty0 2519--2535,
  2000.

\bibitem[Madry et~al.(2017)Madry, Makelov, Schmidt, Tsipras, and
  Vladu]{madry2017towards}
Madry, A., Makelov, A., Schmidt, L., Tsipras, D., and Vladu, A.
\newblock Towards deep learning models resistant to adversarial attacks.
\newblock \emph{arXiv preprint arXiv:1706.06083}, 2017.

\bibitem[Millidge et~al.(2020)Millidge, Tschantz, and
  Buckley]{millidge2020predictive}
Millidge, B., Tschantz, A., and Buckley, C.~L.
\newblock Predictive coding approximates backprop along arbitrary computation
  graphs.
\newblock \emph{arXiv preprint arXiv:2006.04182}, 2020.

\bibitem[Moraitis et~al.(2020)Moraitis, Sebastian, and
  Eleftheriou]{moraitis2020shortterm}
Moraitis, T., Sebastian, A., and Eleftheriou, E.
\newblock Short-term synaptic plasticity optimally models continuous
  environments, 2020.

\bibitem[Nessler et~al.(2009)Nessler, Pfeiffer, and Maass]{nessler2009stdp}
Nessler, B., Pfeiffer, M., and Maass, W.
\newblock Stdp enables spiking neurons to detect hidden causes of their inputs.
\newblock \emph{Advances in neural information processing systems},
  22:\penalty0 1357--1365, 2009.

\bibitem[Nessler et~al.(2013)Nessler, Pfeiffer, Buesing, and
  Maass]{nessler2013PLoS}
Nessler, B., Pfeiffer, M., Buesing, L., and Maass, W.
\newblock Bayesian computation emerges in generic cortical microcircuits
  through spike-timing-dependent plasticity.
\newblock \emph{PLoS computational biology}, 9\penalty0 (4):\penalty0 e1003037,
  2013.

\bibitem[N{\o}kland(2016)]{nokland2016direct}
N{\o}kland, A.
\newblock Direct feedback alignment provides learning in deep neural networks.
\newblock \emph{Advances in neural information processing systems}, 29, 2016.

\bibitem[Oja(1982)]{oja1982simplified}
Oja, E.
\newblock Simplified neuron model as a principal component analyzer.
\newblock \emph{Journal of mathematical biology}, 15\penalty0 (3):\penalty0
  267--273, 1982.

\bibitem[Olshausen \& Field(1996)Olshausen and Field]{olshausen1996emergence}
Olshausen, B.~A. and Field, D.~J.
\newblock Emergence of simple-cell receptive field properties by learning a
  sparse code for natural images.
\newblock \emph{Nature}, 381\penalty0 (6583):\penalty0 607--609, 1996.

\bibitem[Olshausen \& Field(1997)Olshausen and Field]{olshausen1997sparse}
Olshausen, B.~A. and Field, D.~J.
\newblock Sparse coding with an overcomplete basis set: A strategy employed by
  v1?
\newblock \emph{Vision research}, 37\penalty0 (23):\penalty0 3311--3325, 1997.

\bibitem[Payeur et~al.(2021)Payeur, Guerguiev, Zenke, Richards, and
  Naud]{payeur2021burst}
Payeur, A., Guerguiev, J., Zenke, F., Richards, B.~A., and Naud, R.
\newblock Burst-dependent synaptic plasticity can coordinate learning in
  hierarchical circuits.
\newblock \emph{Nature neuroscience}, 24\penalty0 (7):\penalty0 1010--1019,
  2021.

\bibitem[Pehlevan \& Chklovskii(2015)Pehlevan and Chklovskii]{PehlevanNIPS2015}
Pehlevan, C. and Chklovskii, D.
\newblock A normative theory of adaptive dimensionality reduction in neural
  networks.
\newblock In Cortes, C., Lawrence, N., Lee, D., Sugiyama, M., and Garnett, R.
  (eds.), \emph{Advances in Neural Information Processing Systems}, volume~28.
  Curran Associates, Inc., 2015.
\newblock URL
  \url{https://proceedings.neurips.cc/paper/2015/file/861dc9bd7f4e7dd3cccd534d0ae2a2e9-Paper.pdf}.

\bibitem[Pehlevan \& Chklovskii(2014)Pehlevan and
  Chklovskii]{pehlevan2014hebbian}
Pehlevan, C. and Chklovskii, D.~B.
\newblock A hebbian/anti-hebbian network derived from online non-negative
  matrix factorization can cluster and discover sparse features.
\newblock In \emph{2014 48th Asilomar Conference on Signals, Systems and
  Computers}, pp.\  769--775. IEEE, 2014.

\bibitem[Pehlevan et~al.(2017)Pehlevan, Genkin, and
  Chklovskii]{pehlevan2017clustering}
Pehlevan, C., Genkin, A., and Chklovskii, D.~B.
\newblock A clustering neural network model of insect olfaction.
\newblock In \emph{2017 51st Asilomar Conference on Signals, Systems, and
  Computers}, pp.\  593--600. IEEE, 2017.

\bibitem[Pfeiffer \& Pfeil(2018)Pfeiffer and Pfeil]{pfeiffer2018deep}
Pfeiffer, M. and Pfeil, T.
\newblock Deep learning with spiking neurons: opportunities and challenges.
\newblock \emph{Frontiers in neuroscience}, 12:\penalty0 774, 2018.

\bibitem[Pogodin \& Latham(2020)Pogodin and Latham]{pogodin2020kernelized}
Pogodin, R. and Latham, P.~E.
\newblock Kernelized information bottleneck leads to biologically plausible
  3-factor hebbian learning in deep networks.
\newblock \emph{arXiv preprint arXiv:2006.07123}, 2020.

\bibitem[Pogodin et~al.(2021)Pogodin, Mehta, Lillicrap, and
  Latham]{pogodin2021towards}
Pogodin, R., Mehta, Y., Lillicrap, T.~P., and Latham, P.~E.
\newblock Towards biologically plausible convolutional networks.
\newblock \emph{arXiv preprint arXiv:2106.13031}, 2021.

\bibitem[Poirazi \& Papoutsi(2020)Poirazi and
  Papoutsi]{poirazi2020illuminating}
Poirazi, P. and Papoutsi, A.
\newblock Illuminating dendritic function with computational models.
\newblock \emph{Nature Reviews Neuroscience}, 21\penalty0 (6):\penalty0
  303--321, 2020.

\bibitem[Qin et~al.(2020)Qin, Frosst, Raffel, Cottrell, and
  Hinton]{qin2020deflecting}
Qin, Y., Frosst, N., Raffel, C., Cottrell, G., and Hinton, G.
\newblock Deflecting adversarial attacks.
\newblock \emph{arXiv preprint arXiv:2002.07405}, 2020.

\bibitem[Radford et~al.(2015)Radford, Metz, and
  Chintala]{radford2015unsupervised}
Radford, A., Metz, L., and Chintala, S.
\newblock Unsupervised representation learning with deep convolutional
  generative adversarial networks.
\newblock \emph{arXiv preprint arXiv:1511.06434}, 2015.

\bibitem[Rauber et~al.(2017)Rauber, Brendel, and Bethge]{rauber2017foolbox}
Rauber, J., Brendel, W., and Bethge, M.
\newblock Foolbox: A python toolbox to benchmark the robustness of machine
  learning models.
\newblock \emph{arXiv preprint arXiv:1707.04131}, 2017.

\bibitem[Rutishauser et~al.(2011)Rutishauser, Douglas, and
  Slotine]{rutishauser2011collective}
Rutishauser, U., Douglas, R.~J., and Slotine, J.-J.
\newblock Collective stability of networks of winner-take-all circuits.
\newblock \emph{Neural computation}, 23\penalty0 (3):\penalty0 735--773, 2011.

\bibitem[Sanger(1989)]{sanger1989optimal}
Sanger, T.~D.
\newblock Optimal unsupervised learning in a single-layer linear feedforward
  neural network.
\newblock \emph{Neural networks}, 2\penalty0 (6):\penalty0 459--473, 1989.

\bibitem[Sarwat et~al.(2022)Sarwat, Moraitis, Wright, and
  Bhaskaran]{sarwat2022chalcogenide}
Sarwat, S.~G., Moraitis, T., Wright, C.~D., and Bhaskaran, H.
\newblock Chalcogenide optomemristors for multi-factor neuromorphic
  computation.
\newblock \emph{Nature communications}, 13\penalty0 (1):\penalty0 1--9, 2022.

\bibitem[Scellier \& Bengio(2017)Scellier and Bengio]{scellier2017equilibrium}
Scellier, B. and Bengio, Y.
\newblock Equilibrium propagation: Bridging the gap between energy-based models
  and backpropagation.
\newblock \emph{Frontiers in computational neuroscience}, 11:\penalty0 24,
  2017.

\bibitem[Scherr et~al.(2022)Scherr, Guo, and Moraitis]{scherr2022self}
Scherr, F., Guo, Q., and Moraitis, T.
\newblock Self-supervised learning through efference copies.
\newblock \emph{arXiv preprint arXiv:2210.09224}, 2022.

\bibitem[Sejnowski(2020)]{sejnowski2020unreasonable}
Sejnowski, T.~J.
\newblock The unreasonable effectiveness of deep learning in artificial
  intelligence.
\newblock \emph{Proceedings of the National Academy of Sciences}, 117\penalty0
  (48):\penalty0 30033--30038, 2020.

\bibitem[Von~der Malsburg(1973)]{von1973self}
Von~der Malsburg, C.
\newblock Self-organization of orientation sensitive cells in the striate
  cortex.
\newblock \emph{Kybernetik}, 14\penalty0 (2):\penalty0 85--100, 1973.

\bibitem[Xiao et~al.(2017)Xiao, Rasul, and Vollgraf]{xiao2017/online}
Xiao, H., Rasul, K., and Vollgraf, R.
\newblock Fashion-mnist: a novel image dataset for benchmarking machine
  learning algorithms, 2017.

\bibitem[Zador et~al.(2022)Zador, Richards, {\"O}lveczky, Escola, Bengio,
  Boahen, Botvinick, Chklovskii, Churchland, Clopath, et~al.]{zador2022toward}
Zador, A., Richards, B., {\"O}lveczky, B., Escola, S., Bengio, Y., Boahen, K.,
  Botvinick, M., Chklovskii, D., Churchland, A., Clopath, C., et~al.
\newblock Toward next-generation artificial intelligence: Catalyzing the
  neuroai revolution.
\newblock \emph{arXiv preprint arXiv:2210.08340}, 2022.

\end{thebibliography}
\bibliographystyle{icml2022}
\end{document}